\documentclass[sigconf]{acmart}

\copyrightyear{2024}
\acmYear{2024}
\setcopyright{acmlicensed}
\acmConference[KDD '24]{Proceedings of the 30th ACM SIGKDD Conference on Knowledge Discovery and Data Mining}{August 25--29, 2024}{Barcelona, Spain}
\acmBooktitle{Proceedings of the 30th ACM SIGKDD Conference on Knowledge Discovery and Data Mining (KDD '24), August 25--29, 2024, Barcelona, Spain}\acmDOI{10.1145/3637528.3671912}
\acmISBN{979-8-4007-0490-1/24/08}
\usepackage[ruled, lined, vlined, linesnumbered]{algorithm2e}
\usepackage{amsmath, amsthm}

\usepackage{amssymb}
\usepackage{bm}
\usepackage{multirow}
\usepackage{subfig}
\usepackage{threeparttable}
\usepackage{diagbox} 
\usepackage{booktabs}
\usepackage{array}
\usepackage{pdflscape}
\usepackage{graphicx}
\usepackage{xcolor, xspace}
\usepackage{algorithmic}
\usepackage{nccmath}
\usepackage{makecell,rotating}
\theoremstyle{definition}

\newcommand{\ie}{\emph{i.e.},\xspace}

\newcommand\figref[1]{Figure~\ref{#1}}
\newcommand\tabref[1]{Table~\ref{#1}}
\newcommand\secref[1]{Sec.~\ref{#1}}
\newcommand\equref[1]{Eq.(\ref{#1})}

\newcommand\appref[1]{Appendix~\ref{#1}}
\newcommand{\fakeparagraph}[1]{\vspace{1mm}\noindent\textbf{#1.}}
\newcommand{\sysname}{pre-determined GWT\xspace}
\newcommand{\Sysname}{Pre-determined GWT\xspace}

\begin{document}

\title{Pre-Training Identification of Graph Winning Tickets in Adaptive Spatial-Temporal Graph Neural Networks}

\author{Wenying Duan}
\email{wenyingduan@ncu.edu.cn}
\affiliation{%
  \institution{Jiangxi Provincial Key Laboratory of Intelligent Systems and Human-Machine Interaction, Nanchang University}
  \city{Nanchang}
  \country{China}
}

\author{Tianxiang Fang}
\email{6109121076@email.ncu.edu.cn}
\affiliation{%
  \institution{Nanchang University}
  \city{Nanchang}
  \country{China}
}

\author{Hong Rao}
\email{raohong@ncu.edu.cn}
\affiliation{%
 \institution{School of Software\\Nanchang University}
 \city{Nanchang}
 \country{China}
}

\author{Xiaoxi He}\authornote{Corresponding author}
\email{hexiaoxi@um.edu.mo}
\affiliation{%
  \institution{Faculty of Science and Technology\\University of Macau}
  \city{Macau}
  \country{China}
}


\begin{abstract}
In this paper, we present a novel method to significantly enhance the computational efficiency of Adaptive Spatial-Temporal Graph Neural Networks (ASTGNNs) by introducing the concept of the Graph Winning Ticket (GWT), derived from the Lottery Ticket Hypothesis (LTH). By adopting a pre-determined star topology as a GWT prior to training, we balance edge reduction with efficient information propagation, reducing computational demands while maintaining high model performance. Both the time and memory  computational complexity of generating adaptive spatial-temporal graphs is significantly reduced from $\mathcal{O}(N^2)$ to $\mathcal{O}(N)$. Our approach streamlines the ASTGNN deployment by eliminating the need for exhaustive training, pruning, and retraining cycles, and demonstrates empirically across various datasets that it is possible to achieve comparable performance to full models with substantially lower computational costs. Specifically, our approach enables training ASTGNNs on the largest scale spatial-temporal dataset using a single A6000 equipped with 48 GB of memory, overcoming the out-of-memory issue encountered during original training and even achieving state-of-the-art performance. {Furthermore, we delve into the effectiveness of the GWT from the perspective of spectral graph theory, providing substantial theoretical support.} This advancement not only proves the existence of efficient sub-networks within ASTGNNs but also broadens the applicability of the LTH in resource-constrained settings, marking a significant step forward in the field of graph neural networks. Code is available at https://anonymous.4open.science/r/paper-1430.
\end{abstract}



\maketitle
\section{introduction}

Spatial-Temporal Graph Neural Networks (STGNNs) have established themselves as a formidable tool for mining the hidden patterns present in spatial-temporal data, displaying remarkable proficiency in modeling spatial dependencies via graph structures \cite{bib:TNNLS20:Wu}. The construction of these spatial graphs is a pivotal aspect of STGNNs, in which the complex and implicit nature of spatial-temporal relationships has paved the way for the recently emerging self-learned methods that dynamically generate graphs to capture these intricate dependencies in a data-driven manner. Adaptive Spatial-Temporal Graph Neural Networks (ASTGNNs), a state-of-the-art approach to spatial-temporal data processing, are particularly adept at creating adaptive graphs through learnable node embeddings, as exemplified by models such as Graph WaveNet \cite{bib:IJCAI19:Wu} and AGCRN \cite{bib:NIPS20:Bai}.

Despite their advanced performance, ASTGNNs are encumbered by substantial computational overheads during both the training and inference phases, primarily due to the exhaustive calculations required for learning the adaptive adjacency matrices of complete graphs, and the computationally intensive nature of the aggregation phase. This presents a significant challenge when dealing with large-scale spatial-temporal data, where computational efficiency is paramount. Pioneering work \cite{bib:kdd23:Duan} has explored this aspect, improving the efficiency of ASTGNNs during inference via sparsification of the spatial graph. However, the sparsification of the spatial graph relies heavily on the training framework and can only be conducted after the training phase, leaving the efficiency of the training phase itself untouched.

In order to improve the efficiency of both the training and inference phases of ASTGNNs, our research introduces and explores the concept of the \textit{Graph Winning Ticket (GWT)} for the learnable spatial graphs in ASTGNNs, an extension of the Lottery Ticket Hypothesis (LTH) in the context of ASTGNN. The original LTH posits the existence of smaller, efficient sub-networks—'winning tickets'—that can match the performance of the full network with a fraction of the computational cost \cite{frankle2019lottery}. This concept has been extended to the realm of ASTGNNs, where the identification of such sub-networks within the learnable spatial graphs, i.e., GWTs, holds the potential to markedly accelerate the training and inference processes.
However, a simple adoption of the LTH in the context of ASTGNN is not sufficient for practically improving their efficiency during both training and inference phases, as the traditional method of finding winning tickets involves a compute-intensive cycle of training, pruning, and retraining.

In contrast, our work aims to streamline this process by preemptively identifying a GWT for the spatial graph in ASTGNNs. We posit that a star topology, as a spanning tree of the complete graph, serves as an effective pre-determined GWT, striking a balance between edge reduction and efficient information propagation. We argue that the effectiveness of traditional ASTGNNs is enabled by the adoption of a complete spatial graph, which has a diameter of 1 and thus allows for optimally efficient information propagation. However, by relaxing the diameter of the graph from 1 to 2, our star topology significantly trims the number of edges while still preserving the integrity of spatial-temporal communication. We empirically validate the performance of this star topology across various datasets and benchmarks, solidifying its role as a winning ticket for the spatial graphs in ASTGNNs.  

We summarize our main contributions as follows:

\begin{itemize}
    \item To the best of our knowledge, we are the first to improve the efficiency of ASTGNNs during both training and inference phases, with an emphasis on the training phase. By leveraging the concept of the Lottery Ticket Hypothesis (LTH), we posit that an efficient subgraph of ASTGNN's spatial graph can achieve comparable performance to the complete graph with significantly reduced computational overhead. We introduce a star topology as this winning ticket, which is not only sparser but also retains the essential connectivity to ensure effective information propagation across the network. This pre-determined topology obviates the need for the traditional, exhaustive search process involving training, pruning, and retraining, thereby streamlining the deployment of ASTGNNs and substantially improving their efficiency during both the training and inference phases.

    \item Our research also expands the theoretical foundation of the LTH by providing empirical evidence {and substantial theoretical support} for the existence of winning tickets in the spatial graphs of ASTGNNs. The discovery of a pre-determined winning ticket is a significant stride in the application of the LTH, as it demonstrates that such efficient sub-networks can be identified without resorting to the computationally intensive methods traditionally employed. This advance not only reaffirms the LTH within the domain of graph neural networks, but also paves the way for its practical implementation in scenarios where computational resources are limited. By circumventing the need for iterative training and pruning, our approach enhances the feasibility of adopting the LTH in real-world settings, where efficiency and scalability are critical.

    \item We trained two representative ASTGNNs (AGCRN \& Graph Wavenet) with our pre-identified GWTs on five of the largest known spatial-temporal datasets. The performance of the ASTGNNs with the GWTs can match or even surpass that of training with the full spatial graph, and its training and inference costs are drastically smaller. This provides empirical evidence for the existence of winning graph tickets in ASTGNNs, demonstrating that the GWTs identified are stable winning tickets of the spatial graphs within ASTGNNs, highlighting their scalability and superiority.


\end{itemize}
\section{Related Work}
\subsection{Spatial-Temporal Graph Neural Networks}
The analysis of spatial-temporal data necessitates an understanding of dynamic interactions within time-varying signals across spatial domains\cite{wang2024modeling, wu2024earthfarsser}. Spatial-Temporal Graph Neural Networks (STGNNs) are proficient in uncovering latent patterns in these graph-structured data \cite{bib:TNNLS20:Wu}. A key characteristic of STGNNs is their capability to model spatial dependencies among nodes, effectively learning adjacency matrices. Depending on their approach to constructing these matrices, STGNNs can be categorized into pre-defined and self-learned methods.

Pre-defined STGNNs typically employ prior knowledge to construct graphs. For example, ASTGNN \cite{guo2021learning} and STGCN \cite{bib:IJCAI18:Yu} utilize road network structures for graph creation. However, these pre-defined graphs encounter limitations due to their reliance on extensive domain knowledge and the inherent quality of the graph data. Given the implicit and complex nature of spatial-temporal relationships, self-learned methods for graph generation have gained prominence. These methods introduce innovative techniques to capture complex spatial-temporal dependencies, thereby offering significant advantages over traditional pre-defined models.

Self-learned STGNNs can be further divided into two primary categories: feature-based and randomly initialized methods. Feature-based approaches, such as PDFormer \cite{DBLP:conf/aaai/JiangHZW23} and DG \cite{DBLP:journals/isci/PengDLLJWZH21}, construct dynamic graphs from time-variant inputs, enhancing the accuracy of the model. On the other hand, randomly initialized STGNNs, also known as Adaptive Spatial-Temporal Graph Neural Networks (ASTGNNs), facilitate adaptive graph generation through randomly initialized, learnable node embeddings. Graph WaveNet \cite{bib:IJCAI19:Wu} introduced an Adaptive Graph Convolutional Network (AGCN) layer to learn a normalized adaptive adjacency matrix. AGCRN \cite{bib:NIPS20:Bai} further developed this concept with a Node Adaptive Parameter Learning enhanced AGCN (NAPL-AGCN) to discern node-specific patterns. Owing to its remarkable performance, the NAPL-AGCN model has been incorporated into various recent models \cite{DBLP:conf/aaai/Jiang0YJCK0FS23, bib:AAAI22:Choi, bib:ICLR22:Chen}.

Despite the enhanced performance of ASTGNNs, they are burdened with considerable computational overhead. This is primarily due to two factors: \textit{i)} the process of learning an adaptive adjacency matrix necessitates calculating the edge weight between each pair of nodes, and \textit{ii)} the aggregation phase of these networks is inherently computationally intensive. Our research is centered on identifying the \textit{graph winning ticket}—a concept derived from the Lottery Ticket Hypothesis—in order to accelerate training and inference in ASTGNNs. This approach is particularly relevant for handling large-scale spatial-temporal data, where efficiency is crucial.

\subsection{Lottery Ticket Hypothesis.}
The Lottery Ticket Hypothesis (LTH) suggests that within large neural networks, there exist smaller sub-networks (termed "winning tickets") that, when trained in isolation from the start, can reach a similar performance level as the original network in a comparable number of iterations. \cite{frankle2019lottery}
This finding has attracted lots of research attention as it implies the potential of training a much smaller network to reach the accuracy of a dense, much larger network without going through the time and cost-consuming pipeline of fully training the dense network, pruning and then retraining it to restore the accuracy.  The "Early Bird Lottery Ticket" concept builds on the original LTH. It suggests that winning tickets can be identified very early in the training process, much earlier than what was originally proposed in LTH. This finding could further optimize the training of neural networks by allowing significant pruning and resource reduction very early in the training phase.\cite{you2020drawing, chen2021earlybert}. Further, \cite{bib:ICML21:Chen2} generalised LTH to GNNs by iteratively applying UGS to identify graph lottery tickets.  GEBT discovers the existence of graph early-bird tickets \cite{you2022early}.  DGLT generalizes Dual Lottery Ticket Hypothesis (DLTH) to the graph to address information loss and aggregation failure issues caused by sampling-based GNN pruning algorithms \cite{wang2023searching}. However, the pruned GNNs are still hard to generalize to unseen graphs \cite{wang2023snowflake}. RGLT is proposed to find more robust and generalisable GLT to tackle this issue \cite{wang2023brave}. 

For extremely large models and graphs, identifying graph winning tickets typically necessitates a resource-intensive process involving training the network, followed by pruning and retraining. However, our methodology significantly streamlines the deployment of ASTGNNs. It achieves this by obviating the requirement for exhaustive cycles of training, pruning, and retraining.

\section{Preliminaries}
\subsection{Notations and Problem Definition}

\begin{table}[ht]
\centering
\caption{Summary of Notations}
\begin{tabular}{cl}
\toprule
\textbf{Symbol} & \textbf{Description} \\
\midrule
$\mathcal{G}$ & Undirected graph \\
$N$ & Number of nodes in the graph \\
$\mathbf{X}^{t}$ & Feature matrix at time step $t$ \\
$E$ & Learnable node embedding matrix \\
$d$ & Node embedding dimension \\
$\mathcal{V}$ & Set of nodes in a graph \\
$\mathcal{E}$ & Set of edges in a graph \\
$\mathcal{K}_N$ & Complete graph with $N$ nodes \\
$\mathcal{T}$ & Spanning tree \\
$\mathcal{T}^{\star}$ & Star Topology Spanning Tree \\
$e_c$ & Node embedding vector of the central node \\
$\Theta$ & Model parameters \\
$\mathbf{A}$ & Adjacency matrix \\
$\operatorname{GAT}$ & Graph Attention Network function \\
$\mathbf{Z}^{t}$ & Output of the model at time step $t$ \\
\bottomrule
\end{tabular}
\end{table}
Frequently used notations are summarized in Table 6. Following the conventions in spatial-temporal graph neural network researches~\cite{bib:IJCAI18:Yu, bib:AAAI18:Yan, bib:IJCAI:bai, bib:IJCAI20:Huang},
we denote the spatial-temporal data as a sequence of frames: 
$\{\mathbf{X}^{1}, \mathbf{X}^{2}, \ldots$  $, \mathbf{X}^{t}, \ldots \}$,
where a single frame $\mathbf{X}^{t}\in \mathbb{R}^{N \times D}$ 
is the $D$-dimensional data collated from $N$ different locations at time $t$.
For a chosen task time $\tau$, we aim to learn a function mapping the ${{T}_{in}}$ historical observations into the future observations in the next ${T}_{out}$ timesteps:
\begin{equation}
    \mathbf{X}^{(\tau+1): (\tau+{T}_{out})}\xleftarrow{}\mathcal{F}(\mathbf{X}^{(\tau-{T}_{in}+1): \tau})
\end{equation}

\subsection{GAT vs. AGCN}

\fakeparagraph{Graph Attention Network}
Given an undirected graph $\mathcal{G}=\{\mathcal{V},\mathcal{E}\}$, $\mathcal{V}$ is the set of nodes, $\mathcal{E}$ and $\mathbf{X} =\{x_u\}_{u=1}^{N}\in \mathbb{R}^{N\times D}$ is the corresponding set of edges and node features, respectively, where $N=|\mathcal{V}|$ is the number of nodes, $D$ is the feature dimension. The adjacent matrix can be denoted as $\mathbf{A}=[\mathbf{A}_{uv}]$, where $\mathbf{A}_{uv}=1$ if there is an edge $(u,v)\in \mathcal{E}$ and  $\mathbf{A}_{uv}=0$ otherwise. To account for the importance of neighbor nodes in learning graph structure, GAT integrates the attention
mechanism into the node aggregation operation as:
\begin{equation}\label{eq:gat}
\begin{aligned}
{z}_{u}&
=\sum_{v \in \mathcal{N}_{u}} \mathbf{A}_{u v} {x}_{v}\Theta , \\
\mathbf{A}_{u v}&=\frac{\exp (\operatorname{LeakyReLU}(s_{uv}))}{\sum_{k \in \mathcal{N}_{i}} \exp( \operatorname{LeakyReLU}(s_{u k}))},
s_{u v}=a(x_u, x_v).
\end{aligned}
\end{equation}
Here, $\Theta \in \mathbb{R}^{D\times D^{\prime}}$ is the weight matrix,  $a(\cdot, \cdot)$ is the function of computing attention scores. To simplify, we abbreviate GAT as:
\begin{equation}
\begin{aligned}
     \mathbf{Z} &= \mathbf{A}\mathbf{X}\Theta,\\
     \mathbf{A}=&\operatorname{GAT(\mathcal{G, \mathbf{X}})},
\end{aligned}
\end{equation}
where $\mathbf{Z} \in  \mathbb{R}^{N \times D^{\prime}}$, $\operatorname{GAT}(\cdot)$ is the graph attention function. 

\fakeparagraph{Adaptive Graph Convolution Network}
Adaptive Graph Convolutional Network (AGCN) facilitates adaptive learning of graph structures through randomly initialized learnable matrices. This approach lays the groundwork for the evolution of Adaptive Spatial-Temporal Graph Neural Networks (ASTGNNs). Among the notable ASTGNN models are Graph WaveNet and AGCRN. Within the Graph WaveNet framework, the AGCN is characterized as follows.
\begin{equation}\label{eq:gwn}
    \begin{aligned}
        \mathbf{Z}^{t} &=\mathbf{A}\mathbf{X}^t\Theta,\\
        \mathbf{A}=&\operatorname{Softmax}(\operatorname{ReLU}(E_{1}E_{2}^{\top}),
    \end{aligned}
\end{equation}
where $E_{1}\in \mathbb{R}^{N\times d}$ and $E_{2}\in \mathbb{R}^{N\times d}$ are the source node embeddings and target node embeddings, respectively.
While in AGCRN, AGCN is defined as:
\begin{equation}\label{eq:agcrn}
    \begin{aligned}
         \mathbf{Z}^{t} &=\mathbf{A}\mathbf{X}^t\Theta,\\
        \mathbf{A}=& \operatorname{Softmax}(\operatorname{ReLU}(EE^{\top})),
    \end{aligned}
\end{equation}
where $E\in \mathbb{R}^{N\times d}$ is the node embeddings.
\equref{eq:gwn} and \equref{eq:agcrn} are extremely similar in form, with \equref{eq:agcrn} being more concise. Therefore, the general form of AGCN referred to \equref{eq:agcrn} in this paper.
Upon close observation of \equref{eq:agcrn}, it is not difficult to find that AGCN can be reformulated as the following mathematical expression likes GAT:
\begin{equation}
\begin{aligned}
     z^{t}_u &=\sum_{v \in \mathcal{V}} \mathbf{A}_{u v} x^{t}_{v}\Theta,\\
\mathbf{A}_{u v} &=\frac{\exp (\operatorname{ReLU}\left(s_{u v}\right))}{\sum_{k \in \mathcal{V}} \exp(\operatorname{ReLU} \left(s_{u k}\right))},
s_{u v}={e}_{u}{e}^{\top}_{v},
\end{aligned}
\end{equation}
which is similar to \equref{eq:gat}. Thus, AGCN can be considered a special kind of graph attention network on a complete graph with self-loops. We further abbreviate AGCN as the following equations:
\begin{equation}\label{eq:aggat}
    \begin{aligned}
        Z_t &= \mathbf{A}\mathbf{X}_t\Theta,\\
        \mathbf{A}&=\operatorname{GAT}(\mathcal{\tilde{K}}_N,E),
    \end{aligned}
\end{equation}
where  $\mathcal{\tilde{K}}_N$ is the $N$-order complete graph $\mathcal{K}_{N}$ with self-loops. As the diameter of $\mathcal{K}_{N}$ is 1,  
AGCN facilitates the aggregation of information from all nodes to each individual node within $\mathcal{\tilde{K}}_N$. This characteristic significantly enhances the network's capability to model global spatial dependencies, culminating in its state-of-the-art performance in relevant tasks, as documented in \cite{bib:NIPS20:Bai,bib:ICLR22:Chen}.

The model utilizing multi-layers of AGCN for modeling spatial dependencies is designated as ASTGNN (Adaptive Spatio-Temporal Graph Neural Network). The spatial-temporal forecasting problem when addressed using ASTGNN is mathematically expressed as:
\begin{equation}
    \mathbf{X}^{(\tau+1): (\tau+{T}_{out})}\xleftarrow{}\mathcal{F}(\mathbf{X}^{(\tau-{T}_{in}+1): \tau};\theta, \mathcal{\tilde{K}}_N),
\end{equation}
In this formulation, $\mathcal{F}$ represents the forecasting function of ASTGNN parameterised by $\theta$,  which predicts future values $\mathbf{X}^{(\tau+1): (\tau+{T}_{out})}$ based on the input sequence $\mathbf{X}^{(\tau-{T}_{in}+1): \tau}$ and the structural information encoded in the graph $\mathcal{\tilde{K}}_N$.

However, a notable limitation arises during the training phase. The computational complexity associated with calculating adjacency matrices and executing graph convolution operations on complete graphs is of  $\mathcal{O}(N^2)$. This significant computational demand imposes a constraint on the model's scalability, particularly in scenarios involving large spatial-temporal datasets, where reducing computational complexity is crucial for practical applicability.

\subsection{Graph Tickets Hypothesis}
The Graph Tickets Hypothesis represents an extension of the original Lottery Tickets Hypothesis, initially introduced by UGS \cite{bib:ICML21:Chen2}. UGS demonstrated that GWT (\ie compact sub-graphs) are present within randomly initialized GNNs, which can be retrained to achieve performance comparable to, or even surpassing, that of GNNs trained on the original full graph. This finding underscores the potential for efficiency improvements in GNN training methodologies. However, designing a graph pruning method to identify GWTs in ASTGNNs proves to be a nontrivial task. The state-of-the art, AGS demonstrates that spatial graphs in ASTGNNs can undergo sparsification up to 99.5\% with no detrimental impact on test accuracy \cite{bib:kdd23:Duan}. Nonetheless, this robustness to sparsification does not hold uniformly; when ASTGNNs, sparsified beyond 99\%, are reinitialized and retrained on the same dataset, there is a notable and consistent decline in accuracy. This dichotomy underscores the nuanced complexity inherent in finding winning graph tickets in ASTGNNs and calls for further investigation.


\section{Method}

\subsection{Pre-Identifying the Graph Winning  Ticket}
Our objective is to identify a sparse subgraph of the spatial graph pre-training and to train ASTGNNs efficiently on this subgraph without compromising performance. An  ASTGNN with a spatial graph $\mathcal{\hat{G}}$, equipped with $K$-layer AGCN can be formulated as:

\begin{equation}\label{eq:astgnn}
\begin{aligned}
    \mathbf{Z}^{t} =&\ \underbrace{\mathbf{\hat{A}}\cdots (\mathbf{\hat{A}}(\mathbf{\hat{A}}\mathbf{X}^t\Theta_{1})\Theta_{2})\cdots\Theta_{K}},\\
   \mathbf{\hat{A}} &= \operatorname{GAT}(\mathcal{\hat{G}},E),
\end{aligned}
\end{equation}
where $\mathcal{\hat{G}}$ is a sparse subgraph of $\mathcal{{K}}_{N}$. However, employing $\mathcal{\hat{G}}$ alone does not ensure the capability to model global spatial dependencies. To maintain the global spatial modeling ability of AGCN and to train ASTGNNs efficiently, we argue that it is essential to use a spanning tree $\mathcal{T}$ of $\mathcal{\tilde{K}}_{N}$, instead of $\mathcal{\tilde{K}}_{N}$, with a sufficient $K$:

\begin{equation}\label{eq:k-hop}
\begin{aligned}
    \mathbf{Z}^{t} =&\ \underbrace{\mathbf{\tilde{A}}\cdots (\mathbf{\tilde{A}}(\mathbf{\tilde{A}}}_{K} \mathbf{X}^t\Theta_{1})\Theta_{2})\cdots\Theta_{K},\\
   \mathbf{\tilde{A}} &= \operatorname{GAT}(\mathcal{T},E),
\end{aligned}
\end{equation}

To mitigate the risk of excessive parameters and overfitting due to a high number of network layers, it is crucial to minimize $r$ as much as possible.
In light of this, we found that star topology spanning trees (with diameter $r=2$) can function as GWTs. We make two notes on the star topology spanning tree:
\begin{enumerate}
    \item \textbf{Motivation:} GAT is a message-passing network, and AGCN can be viewed as modeling a fully connected GAT, allows any node to communicate globally. A spanning tree $\mathcal{T}_N$, as the minimum connected graph of complete graph, can achieve message passing to all other nodes in the graph by stacking $k$ GAT layers, where k is the diameter of $\mathcal{T}_N$. Our goal is to minimize the computational complexity of ASTGNNs, so it's necessary to minimize $k$. Clearly, $\mathcal{T}_N$ with a diameter of 1 doesn't exist. So we start with $k=2$ to examine the existence of spanning trees and we found that there exist $\mathcal{T}_N$ with a diameter of 2, uniquely forming a star topology. We'll detail this motivation in final version.
    \item \textbf{Theoretical Analysis:} Based on spectral graph theory, if one graph is a $\sigma$-$approximation$ of another, they have similar eigensystems and properties. We can prove that $\mathcal{T}_N$ is an N-approximation of $\mathcal{K}_N$. So $\mathcal{K}_N$ and $\mathcal{T}_N$ have similar properties, allowing $\mathcal{T}_N$ with fewer edges to effectively replace $\mathcal{K}_N$ for learning good representations. The complete proof can be found in \appref{append:graph}.
\end{enumerate}

\fakeparagraph{Hypothesis 1}
Given an N-order complete spatial graph $\mathcal{{K}}_{N}$ of an ASTGNN, we investigate an associated star spanning tree: $\mathcal{T}^{\star} = \left\{\mathcal{V},\mathcal{E}^{\star}\right\} $, where $\mathcal{E}^{\star} = \{ ( u_{c},v)  \mid v\in \mathcal{V} \setminus \left \{ u_{c} \right \} \}$, with $u_{c}$ designated as the central node, $v$ designated as the leaf node. All such $\mathcal{T}^\star$ are Graph Winning Ticket (GWT) for the spatial graph of the corresponding ASTGNN.

To ensure the existence of the associated star spanning tree, we have the following proposition:

\begin{proposition}
In an N-order complete graph $\mathcal{{K}}_{N}$, there exists a graph $\mathcal{T}$ such that $\mathcal{T}$ is a spanning tree of $\mathcal{{K}}_{N}$ and the diameter of $\mathcal{T}$ is 2, and the topology of $\mathcal{T}$ unequivocally satisfies definition of star spanning tree in Hypothesis 1.
\end{proposition}

\begin{figure}[t]   
  \centering            
  \subfloat[A complete graph $\mathcal{{K}}_{N}$.]
  {
      \label{fig:com}\includegraphics[width=0.18\textwidth]{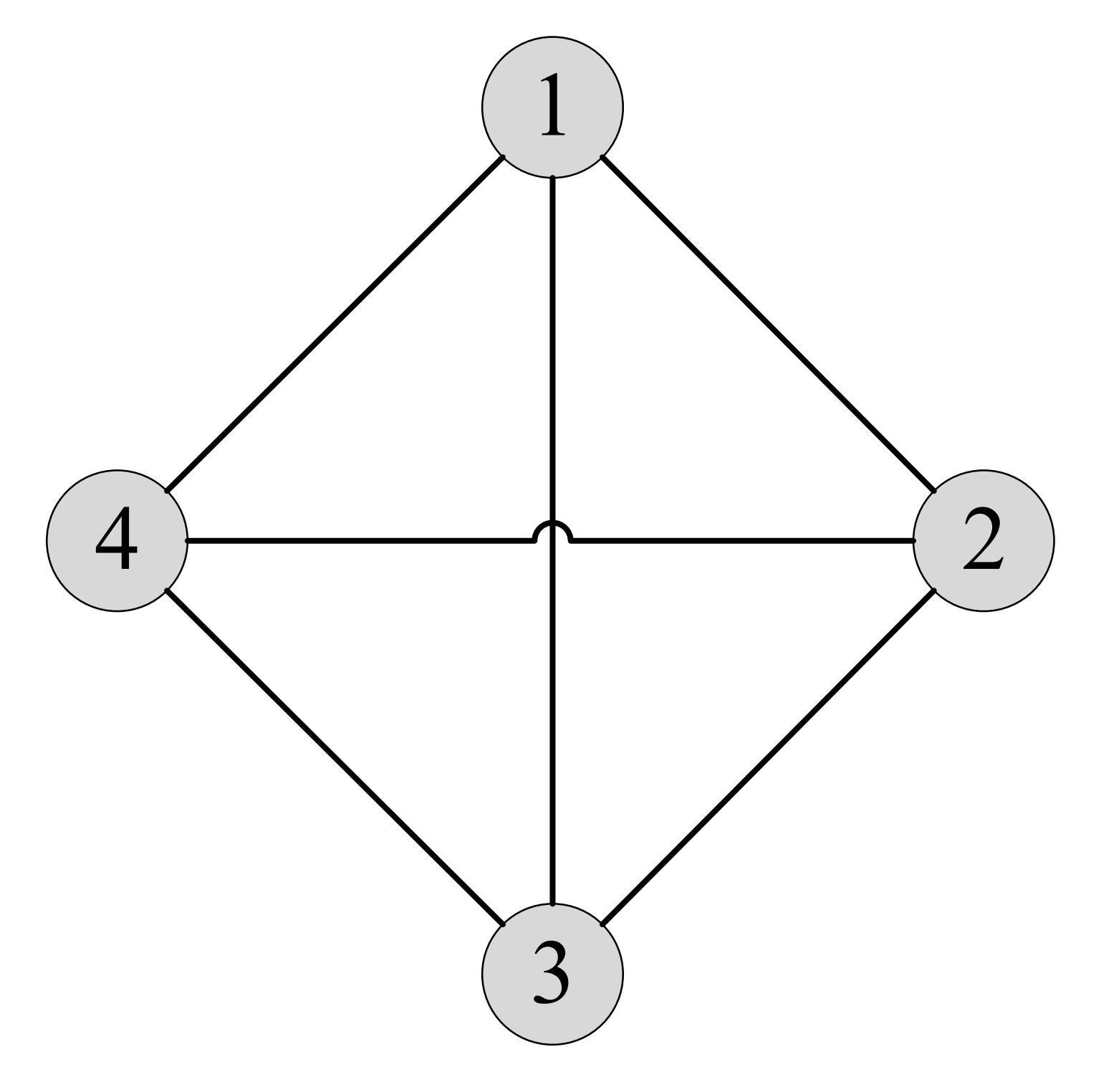}
  }
  \subfloat[A spanning tree of $\mathcal{{K}}_{N}$ with a diameter of 2, identified as a Graph Winning Ticket]
  {
      \label{fig:star}\includegraphics[width=0.18\textwidth]{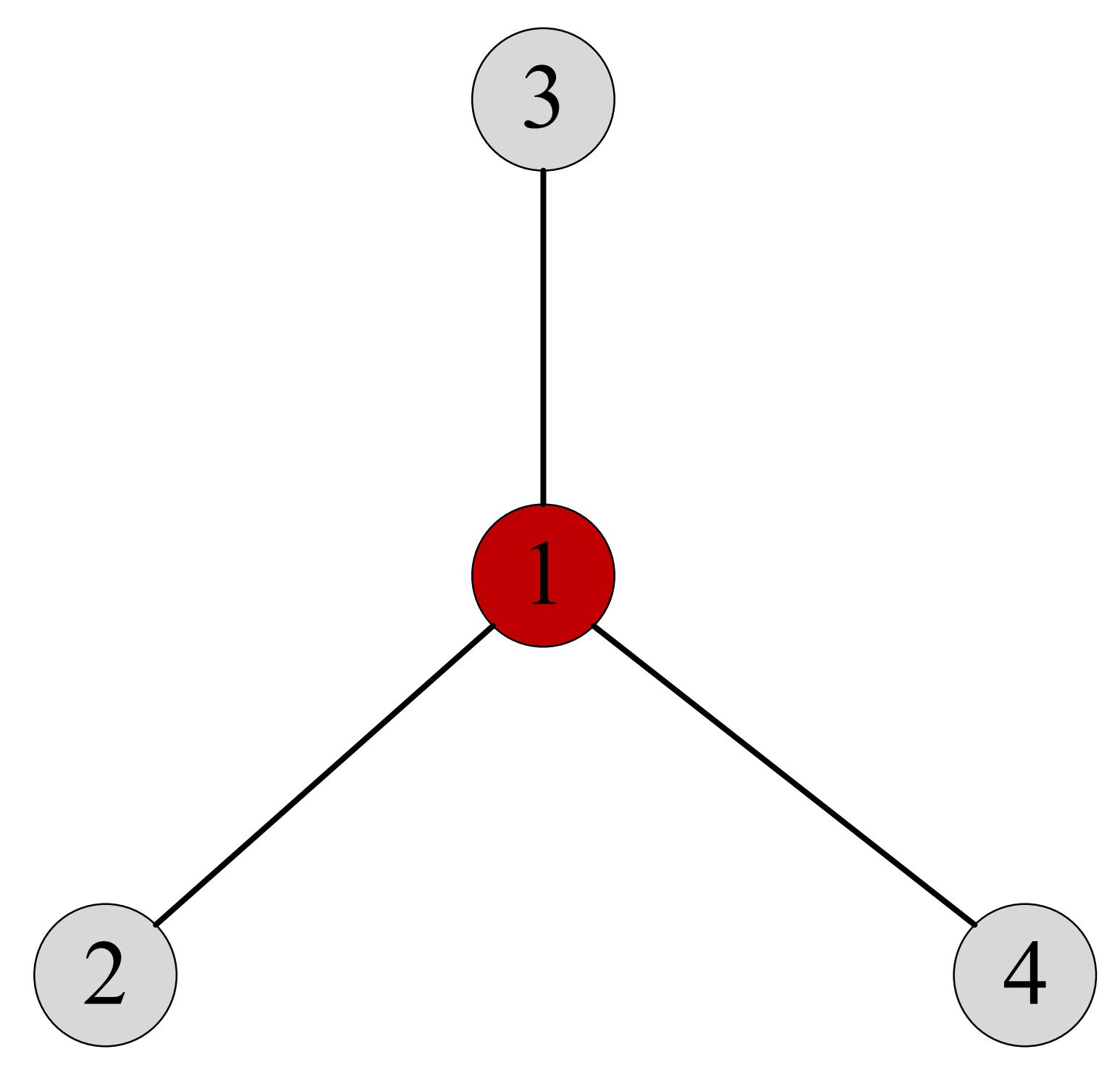}
  }
   \caption{A complete graph and a star spanning tree with a pre-specified node number.}    
  \label{fig:cvs}   
\end{figure}

The proof of Proposition 1 is given in \appref{append:proof}.
To verify Hypothesis 1, we provide empirical evidence demonstrating that such $\mathcal{T}^\star$ are GWTs for their corresponding ASTGNNS in \secref{sec:eval}.
 
 
\fakeparagraph{Sparsity of $\mathcal{T}^\star$} The sparsity of the $\mathcal{T}^\star$ is quantified as $1-\frac{2}{N}$. This represents a significant level of sparsity, particularly as the number of nodes $N$ increases. In such cases, the sparsity becomes increasingly pronounced, highlighting the efficiency of these GWTs in large-scale spatial-temporal datasets.

\subsection{Further Enhancements}\label{sec:eff}
In this section, we discuss two additional enhancements made to training ASTGNNs within $\mathcal{T}^\star$.
In the context of an ASTGNN $\mathcal{F}(\cdot;\theta, \mathcal{\tilde{K}}_{N})$, which comprises multiple AGCN layers, a straightforward approach might involve substituting $\mathcal{\tilde{K}}_{N}$ with $\mathcal{T}^{\star}$ to facilitate rapid training. However, this seemingly intuitive method encounters two primary issues:

\begin{itemize}
\item Efficiency: The method lacks optimal efficiency in training.
\item Central Node Selection: The random selection of the central node $v_c$ could lead to sub-optimal performance.
\end{itemize}

\fakeparagraph{Efficiency} The computational complexity of Graph Neural Network (GNN) training and inference encompasses two primary components: Deep Neural Network (DNN) computation and Graph convolution operation.  Considering the relaxation of the graph's diameter from 1 to 2,
 an ASTGNN necessitates a minimum of two layers of AGCN to maintain comprehensive spatial-temporal communication:
\begin{equation}
\begin{aligned}
    \mathbf{Z}^{t} &=\mathbf{{A^{\star}}}(\mathbf{A^{\star}}\mathbf{X^{t}}\Theta_{1})\Theta_{2},\\
     \mathbf{{A^{\star}}} &=\operatorname{GAT}(\mathcal{T}^{\star},E),
\end{aligned}\label{eq:2layer}
\end{equation}
To ameliorate the computational complexity of Equation \eqref{eq:2layer} in terms of DNN computation, we introduce a streamlined formulation by excluding the parameter $\Theta_{2}$. This modification facilitates 2-hop message passing within a singular AGCN layer, thereby providing the ability to model the global spatial dependencies:
\begin{equation}
\begin{aligned}
    \mathbf{Z}^{t} &=\mathbf{{A^{\star}}}(\mathbf{A^{\star}}\mathbf{X}^{t}\Theta),\\
     \mathbf{{A^{\star}}} &=\operatorname{GAT}(\mathcal{T}^{\star},E),
\end{aligned}\label{eq:aaa}
\end{equation}
From the perspective of graph convolution operations, \eqref{eq:2layer} exhibits informational redundancy in its message-passing process. The message-passing trajectory delineated in \ref{fig:mp} reveals that the paths $u_{c} \rightarrow v$ and $v \rightarrow u_{c}$ are executed twice, engendering superfluous aggregation. Such redundancy could potentially impede the model's efficiency.
\begin{figure}[t]   
  \centering            
  \subfloat[Message passing path of central node's feature.]
  {
      \label{fig:center}\includegraphics[width=0.17\textwidth]{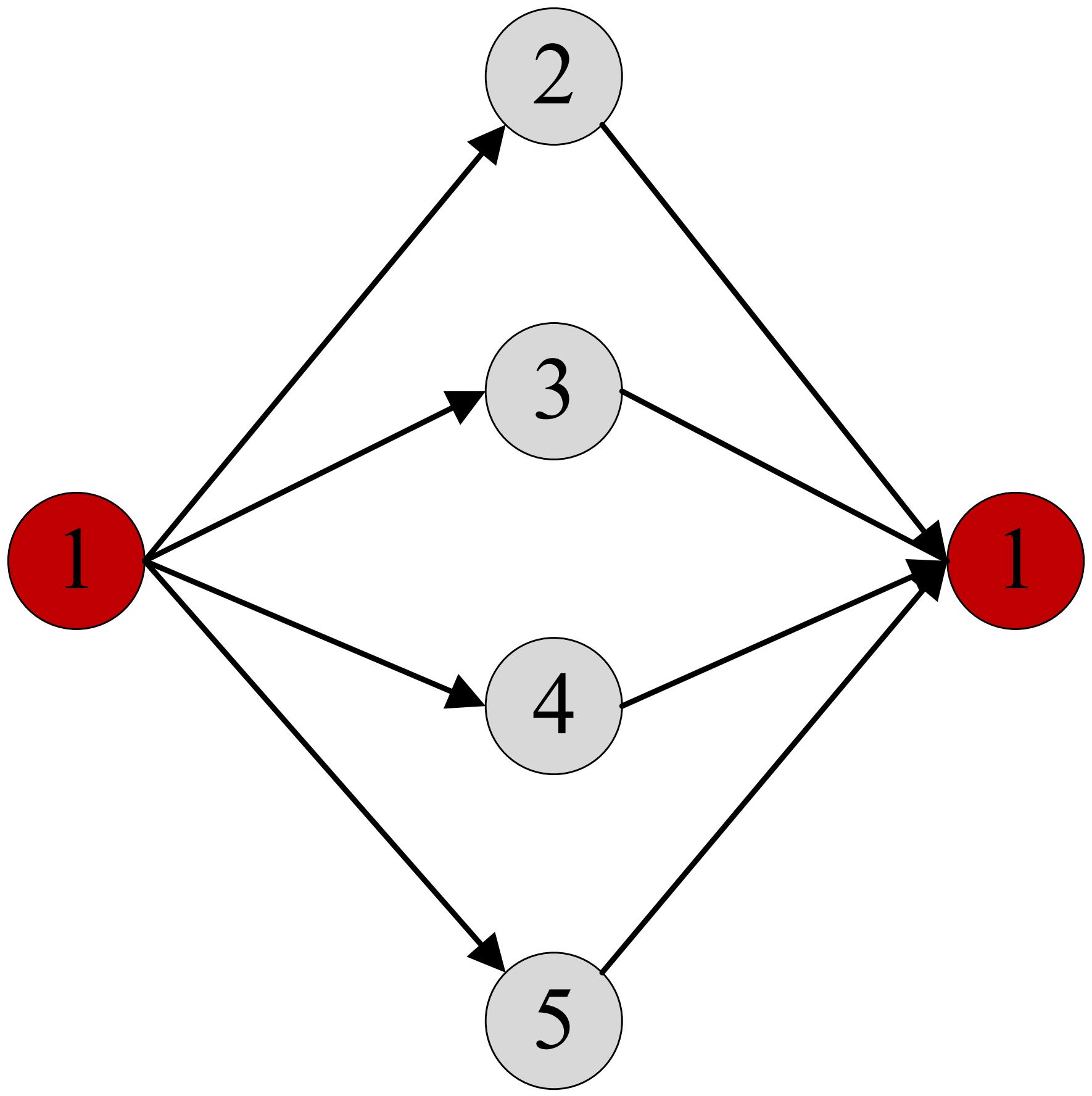}
  }
  \subfloat[Message passing path of leaf node's feature.]
  {
      \label{fig:leaf}\includegraphics[width=0.17\textwidth]{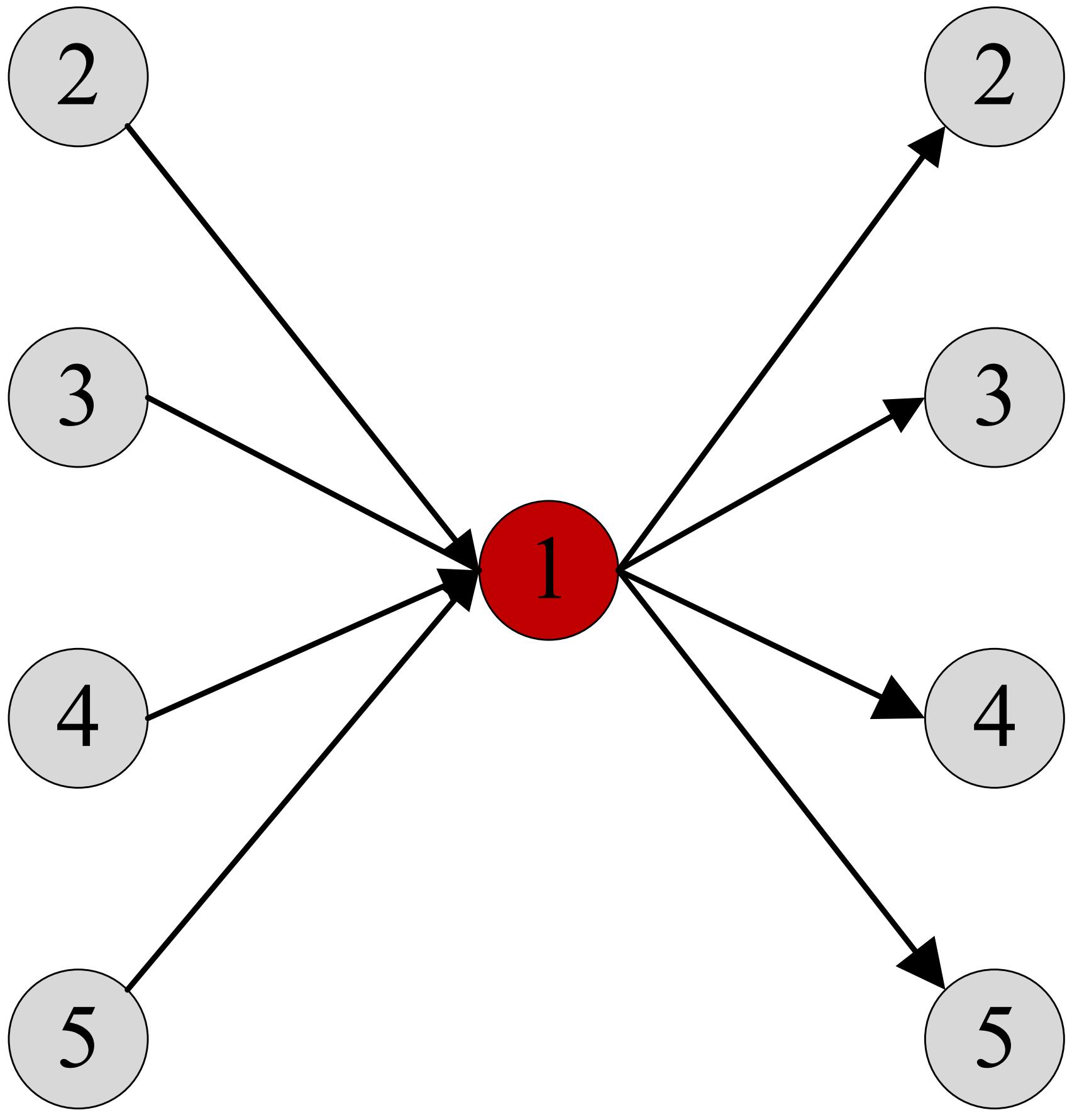}
  }
   \caption{2-hop message passing path of $\mathcal{T}^\star$ with pre-specified node numbers. The red node is the central node $u_v$ and the gray nodes are leaf nodes  $v\in \{ \mathcal{V} \setminus \left \{ u_{c} \right \} \}$.}    
  \label{fig:mp}   
\end{figure}
We therefore perform message passing as illustrated in \figref{fig:leaf} using two directed graphs, denoted as $\overleftarrow{\mathcal{T}^{\star}}$ and $\overrightarrow{\mathcal{T}^{\star}}$. This process can be expressed by the following equations:
\begin{equation}\label{eq:ul}
\begin{aligned}
    \mathbf{Z}^{t} &=\mathbf{U}^{\star}(\mathbf{L}^{\star}\mathbf{X}^{t}\Theta),\\
    \mathbf{L}^{\star}& =\operatorname{GAT}(\overleftarrow{\mathcal{T}^{\star}}, E),\\
    \mathbf{U}^{\star}& =\operatorname{GAT}(\overrightarrow{\mathcal{T}^{\star}}, E),
\end{aligned}
\end{equation}
here, $\overleftarrow{\mathcal{T}^{\star}}= \{\mathcal{V},\overleftarrow{\mathcal{E}} \}$, where $\overleftarrow{\mathcal{E}} ={<v, u_{c}> \mid v\in \mathcal{V}\setminus \{u_{c}\}}$. $\overrightarrow{\mathcal{T}^{\star}}= \{\mathcal{V},\overrightarrow{\mathcal{E}} \}$, where $\overrightarrow{\mathcal{E}} ={<u_{c}, v> \mid v\in \mathcal{V}\setminus \{u_{c}\}}$. 
The computational complexity of graph convolution operations experiences a notable reduction in \equref{eq:ul}. 
To elaborate, the complexity in \equref{eq:2layer} is $\mathcal{O}(2N)$, whereas it is diminished to 
$\mathcal{O}(N)$ in \equref{eq:ul}. 

Despite this enhancement, \equref{eq:2layer} still faces limitations in terms of hardware compatibility. At the hardware level, graph convolution operations are intrinsically linked to the sparse and irregular nature of graph structures. This characteristic might not be compatible with certain hardware architectures, leading to an increased frequency of random memory accesses and limited opportunities for data reuse. Consequently, this can result in significantly higher inference latency for graph convolutions when compared to other neural network architectures.
Then, we introduce a self-loop to the central node $u_{c}$ in both $\overleftarrow{\mathcal{T}^{\star}}$ and $\overrightarrow{\mathcal{T}^{\star}}$. Consequently, we reformulate \equref{eq:ul} to a network namely GWT-AGCN as follows:
\begin{equation}\label{eq:aagcn}
    \mathbf{Z}^{t} = \operatorname{Softmax}\left(\operatorname{ReLU}\left(E e_{c}^{\top}\right)\right) \operatorname{Softmax}\left(\operatorname{ReLU}\left(e_{c} E^{\top}\right)\right) \mathbf{X}^{t}\Theta
\end{equation}
Here, $e_{c} \in \mathbb{R}^{1 \times d}$ represents the node embedding vector of node $u_{c}$. This  GWT-AGCN  layer can serve as an alternative to the AGCN layer in constructing ASTGNNs.


The advantages of \equref{eq:aagcn} are manifold: The equation solely comprises matrix multiplication and standard activation functions, thereby enhancing its compatibility with hardware. In contrast to \equref{eq:ul},
the complexity increased by only $\mathcal{O}(2)$, a change that can be considered inconsequential.

\fakeparagraph{Central Node Selection}
Owing to the non-uniqueness of $\mathcal{T}^{\star}$ in the complete graph $\mathcal{K}_{N}$, directly employing $\mathcal{T}^{\star}$ for training ASTGNNs presents the challenge of central node selection. Viewed through the lens of AGCN, the random selection of a node $u_{c}$ from the vertex set $\mathcal{V}$ is analogous to initializing the node embedding $e_{c}$ randomly.
This approach, however, might introduce bias in the construction of the adaptive graph. To ensure that the selected central node embedding vector $e_{c}$ is positioned at the physical center of the node embedding space $E$, we opt for a setting where $e_{c} = \operatorname{Mean}(E)$, a technique we refer to as averaged initialization. We empirically show that such operation provides better on the prediction accuracy (see \secref{sec:anal}).

\section{EVALUATION}\label{sec:eval}
\begin{table}[t]
  \centering
  \caption{Spatial-temporal datasets statistics.}
    \begin{tabular}{llll}
    \Xhline{1pt}
     \textbf{Dataset} & \textbf{\#Nodes} & \textbf{Time Range} \\
    \hline
    \hline
    PEMS07 & 883   & 05/01/2017-08/06/2017 \\
          CA   & 8,600 & 01/01/2017-12/31/2021 \\
          GLA   & 3,834 & 01/01/2019-12/31/2019 \\
          GBA   & 2,352 & 01/01/2019-12/31/2019 \\
          SD    & 716   & 01/01/2019-12/31/2019 \\
    \Xhline{1pt}
    \end{tabular}
  \label{tab:datasets}
\end{table}
\begin{table*}[htbp]
  \centering
  \small
  \caption{ Performance comparisons. We bold the best results. Standard deviations are suppressed for the sake of room.}\label{tab:results}
    \begin{tabular}{cccccccccccccc}
    \Xhline{1pt}
    \multirow{2}[4]{*}{Data} & \multirow{2}[4]{*}{Method} 
    & \multicolumn{3}{c}{Horizon3} & \multicolumn{3}{c}{Horizon6} & \multicolumn{3}{c}{Horizon12} & \multicolumn{3}{c}{Average} \\
\cmidrule{3-14}          &             & MAE   & RMSE  & \multicolumn{1}{c|}{MAPE} & MAE   & RMSE  & \multicolumn{1}{c|}{MAPE} & MAE   & RMSE  & \multicolumn{1}{c|}{MAPE} & MAE   & RMSE  & MAPE \\
    \hline
    \hline
    \multirow{4}[2]{*}{PEMS07}
        &AGCRN&19.31&31.68&8.18\%&20.70&34.52&8.66\%&22.74&37.94&9.71\%&20.64&34.39&8.74\%\\
        &AGCRN$^\star$&19.36&\textbf{31.56}&8.11\%&20.67&\textbf{33.95}&8.63\%&22.76&\textbf{37.32}&9.59\%&20.67&\textbf{33.95}& 8.67\%\\
        &AGCRN$^\ast$&\textbf{19.27}&31.83& \textbf{8.09}\%&\textbf{20.57}&34.41&\textbf{8.57}\%&\textbf{22.63}&37.97&\textbf{9.39}\%&\textbf{20.57}&34.42&\textbf{8.59}\%\\
         \cline{2-14}
          &GWNET&18.69&30.69&8.02\%&20.26&33.37&\textbf{8.56}\%&22.79&37.11&9.73\%&20.25&33.32&8.63\%\\
          &GWNET$^\star$&18.64&30.61&\textbf{8.01}\%&20.36&33.57&8.68\%&22.39&36.61&\textbf{9.51}\%&20.09&33.13&\textbf{8.59}\%\\
         &GWNET$^\ast$&\textbf{18.48}&\textbf{30.42}&8.20\%&\textbf{19.91}&\textbf{32.98}&8.71\%&\textbf{21.81}&\textbf{36.21}&9.72\% &\textbf{19.80}&\textbf{32.84}&8.62\%\\

    \hline
    \hline
    \multirow{6}[2]{*}{SD}
         & AGCRN&15.71&27.85&11.48\%&\textbf{18.06}&31.51&13.06\%&\textbf{21.86}&39.44&16.52\%&\textbf{18.09}&32.01 &13.28\%\\
        &AGCRN$^\star$&16.06&28.56&11.59\%&18.32&31.65&12.41\%       &22.67&39.06&16.13\%&18.56&31.89&13.19\%\\
        &AGCRN$^\ast$&\textbf{15.49}&\textbf{25.89}&\textbf{10.55}\%&18.12&\textbf{30.96}&\textbf{12.13\%}       &22.24&\textbf{38.79}&\textbf{15.36\%}&18.13&\textbf{30.92}&\textbf{12.41\%}\\
    \cline{2-14}    &GWNET&15.24&25.13&\textbf{9.86}\%&17.74&29.51&\textbf{11.70}\%&\textbf{21.56}&36.82&\textbf{15.13}\%&\textbf{17.74}&29.62 &\textbf{11.88}\%\\
    &GWNET$^\star$&\textbf{15.19}&\textbf{24.97}&10.09\%&\textbf{17.39}&\textbf{29.01}&11.82\%&21.61&\textbf{36.55}&15.35\%&17.90&\textbf{29.56}&13.01\%\\
    &GWNET$^\ast$&15.24&25.07&10.71\%&17.87&29.18&12.21\%&21.94&36.69&15.17\%&17.89&30.06&12.57\%\\
    \hline
    \hline
    \multirow{6}[2]{*}{GBA} 
        
    &AGCRN&18.31&30.24&14.27\%&21.27&34.72&16.89\%&24.85&40.18&20.80\%&21.01&34.25 &16.90\%\\

    &AGCRN$^\star$&18.24&30.18&14.09\%&21.27&34.37&16.77\%&24.82&\textbf{39.68}&20.80\%&20.71&33.75 &15.93\%\\
    &AGCRN$^\ast$&\textbf{17.62}&\textbf{29.49}&\textbf{12.99}\%&\textbf{20.73}&\textbf{34.25}&\textbf{15.65}\%&\textbf{24.72}&40.37&\textbf{18.72}\%&\textbf{20.51}&\textbf{33.57}&\textbf{15.33}\%\\
          \cline{2-14} 
    &GWNET&17.85&29.12&13.92\%&21.11&33.69&17.79\%&25.58&\textbf{40.19}&\textbf{23.48}\%&\textbf{20.91}&33.41 &\textbf{17.66}\%\\
    &GWNET$^\star$&\textbf{17.72}&28.98&13.86\%&\textbf{20.95}&33.47&17.81\%&\textbf{25.55}&41.02&23.68\%&20.96&34.16 &17.79\%\\
    
    & GWNET$^\ast$   &17.76&\textbf{28.97}&\textbf{13.53}\%&21.09&\textbf{33.46}&\textbf{17.58}\%&26.01&40.63&24.11\%&21.01&\textbf{33.38}&\textbf{17.66}\%\\
    \hline
    \hline    
    \multirow{6}[2]{*}{GLA} 
    &AGCRN&17.27&29.70&\textbf{10.78}\%&20.38&34.82&12.70\%&24.59&42.59&16.03\%&20.25&34.84& 12.87\%\\

    &AGCRN$^\star$&17.12&28.10&10.81\%&20.42&33.02&\textbf{11.77}\%&\textbf{24.48}&42.33&15.83\%&20.15&\textbf{32.57}& 12.17\%\\
    &AGCRN$^\ast$&\textbf{16.85}&\textbf{27.38}&10.83\%&\textbf{20.13}&\textbf{32.90}&11.84\%&24.71&\textbf{40.36}&\textbf{14.85}\%& \textbf{20.04}&32.72&\textbf{12.03}\%\\
    \cline{2-14}    &GWNET&17.28&27.68&\textbf{10.18}\%&21.31&33.70&\textbf{13.02}\%&26.99&42.51&17.64\%&21.20&33.58&13.18\%\\

    &GWNET$^\star$&17.18&27.32&10.65\%&\textbf{21.0}&\textbf{33.00}&13.29\%&\textbf{26.39}&42.01&\textbf{17.05}\%&21.10&\textbf{32.97}&13.34\%\\
    
  & GWNET$^\ast$&\textbf{17.16}&\textbf{27.15}&10.87\%&21.23&33.17&13.12\%& 26.41&\textbf{41.95}&17.10\%&\textbf{21.08}&33.03&\textbf{13.13}\%\\
    \hline
    \hline
    \multirow{6}[2]{*}{CA} 
    & AGCRN & --&--&--&--&--&--&--&--& --&--&--&--\\
          
    & AGCRN$^\star$ &16.56&\textbf{26.88}&\textbf{11.93}\%       & 20.13&\textbf{31.87}&15.25\%&\textbf{24.59}&39.65&\textbf{19.86}\%    &\textbf{19.77}&\textbf{31.79}&15.21\%\\
    & AGCRN$^\ast$ &\textbf{16.44}&26.97&12.05\%       & \textbf{19.90}&32.20&\textbf{15.15}\%&24.78&\textbf{39.55}&19.90\%    &19.78&31.98&\textbf{15.19}\%\\
    \cline{2-14}      &GWNET&\textbf{17.14}&\textbf{27.81}&12.62\%&21.68&34.16&17.14\%&28.58&44.13&24.24\%&21.72&34.20 &17.40\%\\
    &GWNET$^\star$&17.19&28.16&\textbf{10.09}\%&22.03&24.66&13.62\%&\textbf{27.05}&\textbf{40.83}&\textbf{21.93}\%&21.38&33.16&16.58\%\\
    &GWNET$^\ast$&17.44&28.08&10.49\%&\textbf{21.23}&\textbf{24.06}&\textbf{13.12}\%&27.08&41.39&22.33\%&\textbf{21.08}&\textbf{32.90}&\textbf{16.12}\%\\
    \Xhline{1pt}
    \end{tabular}%
  \label{tab:addlabel}%
\end{table*}%

\subsection{Experimental Settings}
In this section, we conduct extensive experiments to validate our Hypothesis 1.

\fakeparagraph{Neural Network Architecture}
We evaluate the existence of GWT on two quintessential ASTGNN architectures: \textbf{AGCRN} and \textbf{Graph WaveNet (GWNET)}. AGCRN  integrates an RNN framework, specifically combining AGCN layers with Gated Recurrent Unit (GRU) layers. The AGCN layers are adept at capturing spatial dependencies, whereas the GRU layers are employed to model the temporal dependencies effectively. Conversely, GWNET represents a CNN-based ASTGNN architecture. It amalgamates AGCN, GCN layers, and dilated 1D convolution networks. Here, both GCN and AGCN layers are instrumental in capturing spatial dependencies, whilst the dilated 1D convolution networks are utilized to model the temporal dependencies. AGCRN$^\star$ and GWENT$^\star$ respectively represent AGCRN and GWNET trained within $\mathcal{T}^\star$, while AGCRN$^\ast$ and  GWNET$^\ast$ represent AGCRN and GWNET with in GWT-AGCN described in \secref{sec:eff}, respectively.

\fakeparagraph{Datasets}
We conduct experiments on five of the largest known spatial-temporal datasets. These include PEMS07, a dataset extensively studied \cite{bib:others01:Chen}, along with SD, GBA, GLA, and CA, which were recently introduced in the LargeST dataset \cite{liu2023largest}. 
\tabref{tab:datasets} summarizes the specifications of the datasets used in our experiments. These datasets were partitioned in a 6:2:2 ratio for training, validation, and testing, respectively. The traffic flow data in PEMS07 is aggregated into 5-minute intervals, whereas for SD, GBA, GLA, and CA, the aggregation occurs in 15-minute intervals. We implemented a 12-sequence-to-12-sequence forecast, adhering to the standard protocol in this research domain. 

\fakeparagraph{Implementation Details}
For all evaluated models, we set the number of training iterations to 100. Other training-related configurations adhere to the recommended settings provided in the respective code repositories. To ensure reproducibility and reliability, experiments were conducted ten times on all datasets, except for CA and GLA. Due to their substantially larger data scales, experiments on CA and GLA were limited to three repetitions. These experiments were performed on an NVIDIA RTX A6000 GPU, equipped with 48 GB of memory.

\fakeparagraph{Metrics}
Our comprehensive evaluation encompasses the following dimensions: 
 \textit{i) Performance:} We assess the forecasting accuracy using three established metrics: Mean Absolute Error (MAE), Root Mean Square Error (RMSE), and Mean Absolute Percentage Error (MAPE), and \textit{ii) Efficiency:} Model efficiency is evaluated in terms of both training and inference wall-clock time. Additionally, the batch size during training is reported, reflecting the models' capability to manage large-scale datasets. We set a maximum batch size limit of 64. If a model is unable to operate with this configuration, we progressively reduce the batch size to the highest possible value that fully utilizes the memory capacity of the A6000 GPU.

\subsection{Main Results}
The experimental results are organised as follows: Test accuracies and efficiency comparisons are reported in \tabref{tab:results} and \tabref{tab:efficiency}, respectively.
We also compare the

We make following observations from \tabref{tab:results} and \tabref{tab:efficiency}:
\begin{itemize}
    \item Graph lottery tickets are existent in ASTGNNs. Specifically, AGCRN$^\bigstar$ and GWENT$^\bigstar$ demonstrate performance that is comparable or even superior across all datasets. These findings indicate that $\mathcal{T}^\star$ is a stable 'graph winning lottery ticket' within ASTGNNs when evaluated on datasets such as PEMS07, SD, GBA, GLA, and CA.
    \item Our proposed approach is demonstrably scalable. The CA dataset presents substantial challenges to existing ASTGNNs, evidenced by AGCRN's inability to operate on it. However, the proposed approach facilitates the training of AGCRN on the CA dataset. This capability not only underscores the scalability of the proposed approach but also its superiority. Conventional pruning-based methods necessitate starting the training process with a complete graph. This approach often leads to their inadequacy in identifying graph lottery tickets in large-scale datasets like CA, a limitation that the proposed approach effectively overcomes.
   \item GWT-AGCN has the potential to be an ideal substitute for  AGCN. In comparison, ASTGNN within GWT-AGCN demonstrates enhanced overall performance, particularly in terms of speed, surpassing its predecessor.

    \item GWT-AGCN significantly accelerates the training and inference of the ASTGNNs. The acceleration is more prominent on AGCRNs against GWNETs because a larger portion of the total computation required by GWNETs is used on their GCN layers.
\end{itemize}

\begin{table*}[htbp]
  \centering
 \small
  \caption{ Efficiency comparisons. BS: batch size set during training. Train: training time (in seconds)
per epoch. Infer: inference time (in seconds) on the validation set. Total: total training time (in hours).}
  
  \resizebox{\textwidth}{!}{
    \begin{tabular}{c|cccc|cccc|cccc|cccc|cccc}
    \Xhline{1pt}
    \multirow{2}[4]{*}{Method} & \multicolumn{4}{c|}{PEMS07}  & \multicolumn{4}{c|}{SD}       & \multicolumn{4}{c|}{GBA}      & \multicolumn{4}{c|}{GLA}      & \multicolumn{4}{c}{CA} \\
    \cline{2-21}          & \multicolumn{1}{c|}{BS} & \multicolumn{1}{c|}{Train} & \multicolumn{1}{c|}{Infer} & \multicolumn{1}{c|}{Total} & \multicolumn{1}{c|}{BS} & \multicolumn{1}{c|}{Train} & \multicolumn{1}{c|}{Infer} & \multicolumn{1}{c|}{Total} & \multicolumn{1}{c|}{BS} & \multicolumn{1}{c|}{Train} & \multicolumn{1}{c|}{Infer} & \multicolumn{1}{c|}{Total} & \multicolumn{1}{c|}{BS} & \multicolumn{1}{c|}{Train} & \multicolumn{1}{c|}{Infer} & \multicolumn{1}{c|}{Total} & \multicolumn{1}{c|}{BS} & \multicolumn{1}{c|}{Train} & \multicolumn{1}{c|}{Infer} & \multicolumn{1}{c}{Total} \\
    \hline
    \hline
    AGCRN &64&131&22&4&64&92&15&3&64&536&83&17&45&1413&245&46&--&--&--&--\\
    AGCRN$^\star$ &64&86&14&3&64&78&13&3&64&222&33&7&45&430&69&14&11&1513&196&47\\
    AGCRN$^\ast$&64&56&10&2&64&55&10&2&64&153&24&4&64&286&46&9&16&1021&132&38\\
    \hline
    GWNet& 64&161&21&5&64&82&12&3&64&483&66&15&64&1028&139&32&44&4105&548&113\\ 
    GWNet$^\star$& 64&138&15&4&64&76&12&3&64&380&55&12&64&818&112&26&50&2583&319&81\\ 
GWNet$^\ast$&64&116&13&4&64&71&11&2&64&327&47&1&64&700&96&19&50&2382&295&74\\
    \Xhline{1pt}
    \end{tabular}%
  }
  \label{tab:efficiency}
\end{table*}

\begin{figure}[t]   
  \centering            
  \subfloat[Loss Curves]
  {
      \label{fig:loss_curve_07}\includegraphics[width=0.23\textwidth]{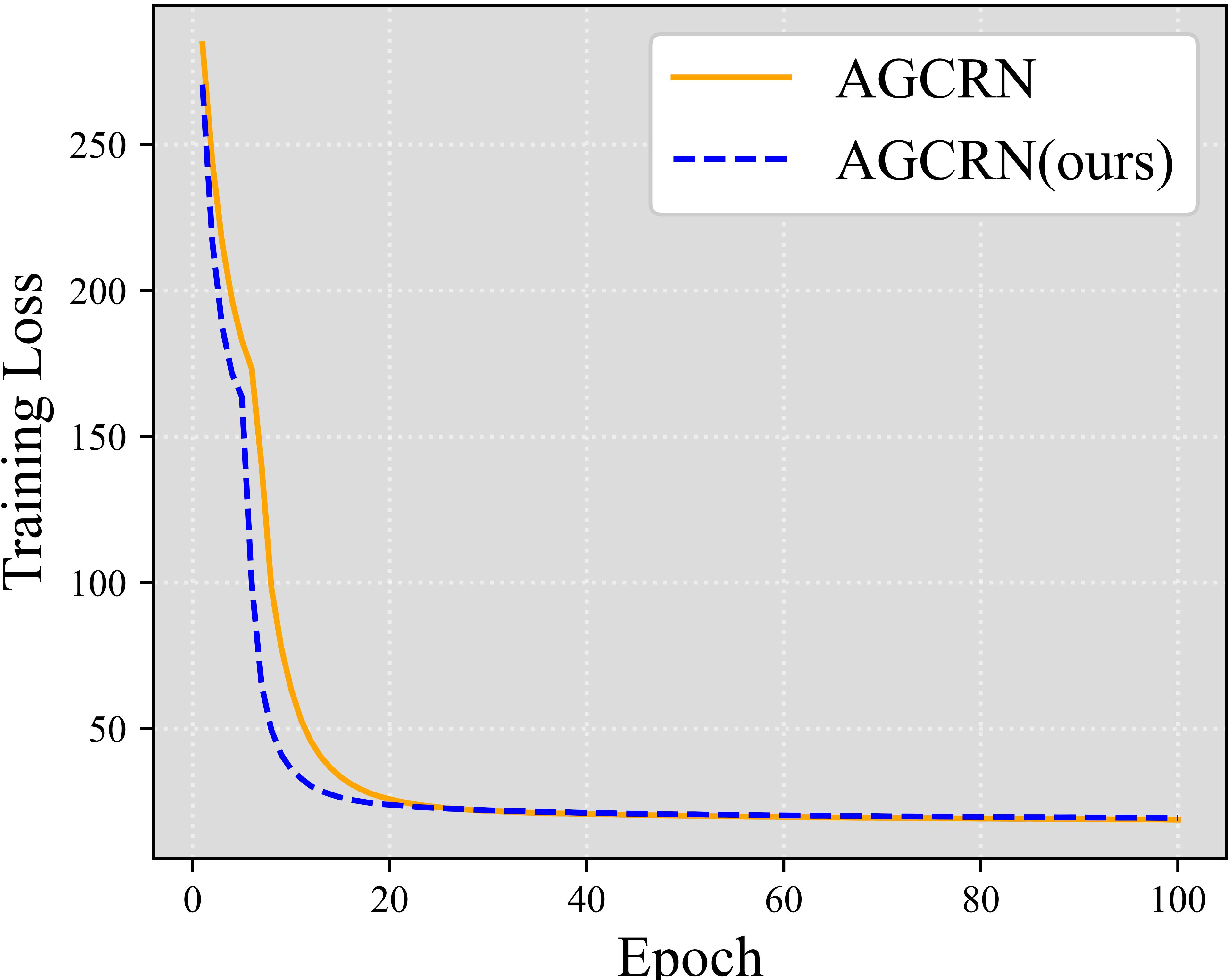}
  }
  \subfloat[MAE Curves]
  {
      \label{fig:mae_curve_07}\includegraphics[width=0.23\textwidth]{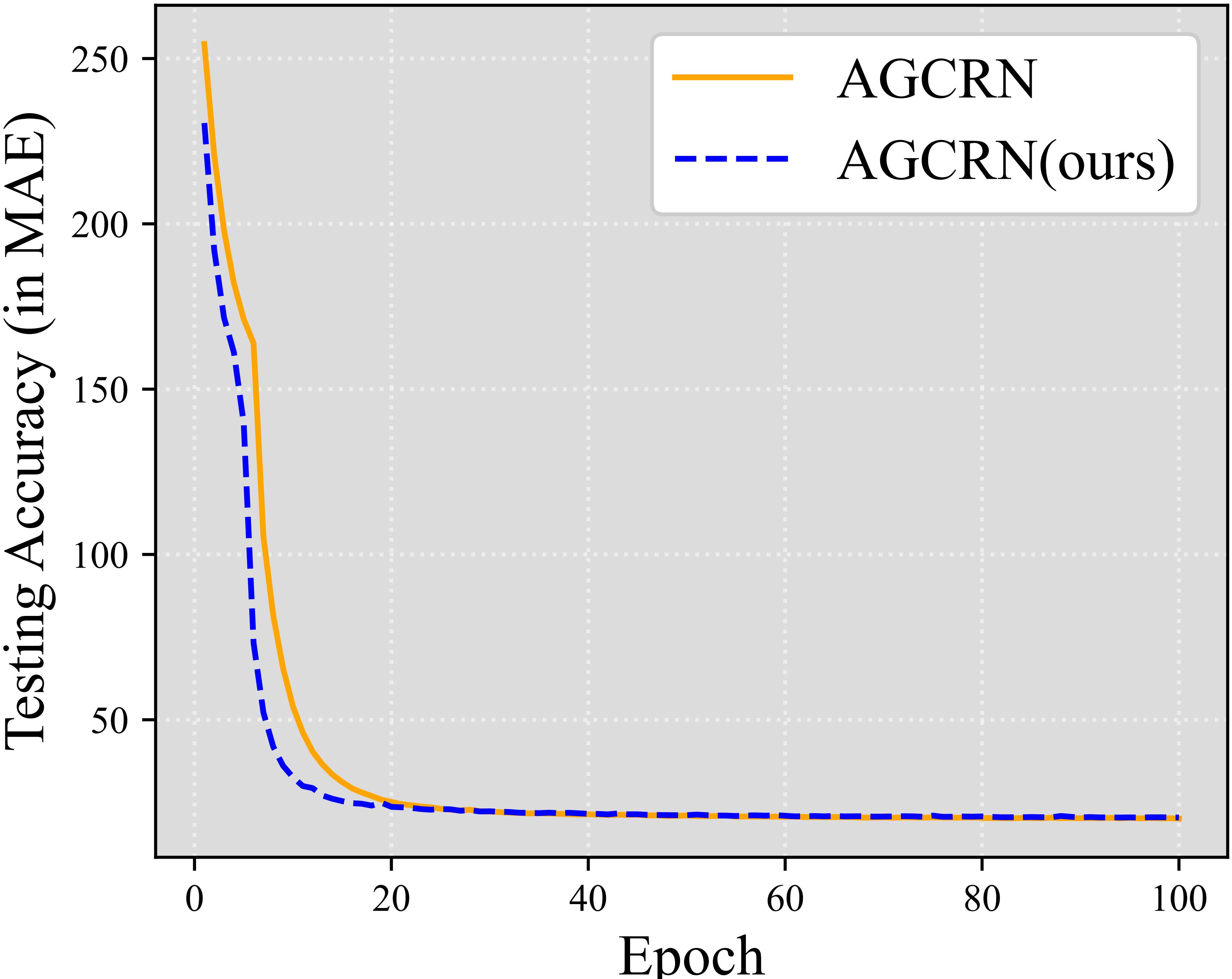}
  }
   \caption{Training loss (a) and testing MAE (b) curve of original AGCRN and AGCRN$^\ast$ trained on PEMS07, respectively.}    
  \label{fig:07}   
\end{figure}

\begin{figure}[t]   
  \centering            
  \subfloat[Loss Curves]
  {
      \label{fig:loss_curve_sd}\includegraphics[width=0.23\textwidth]{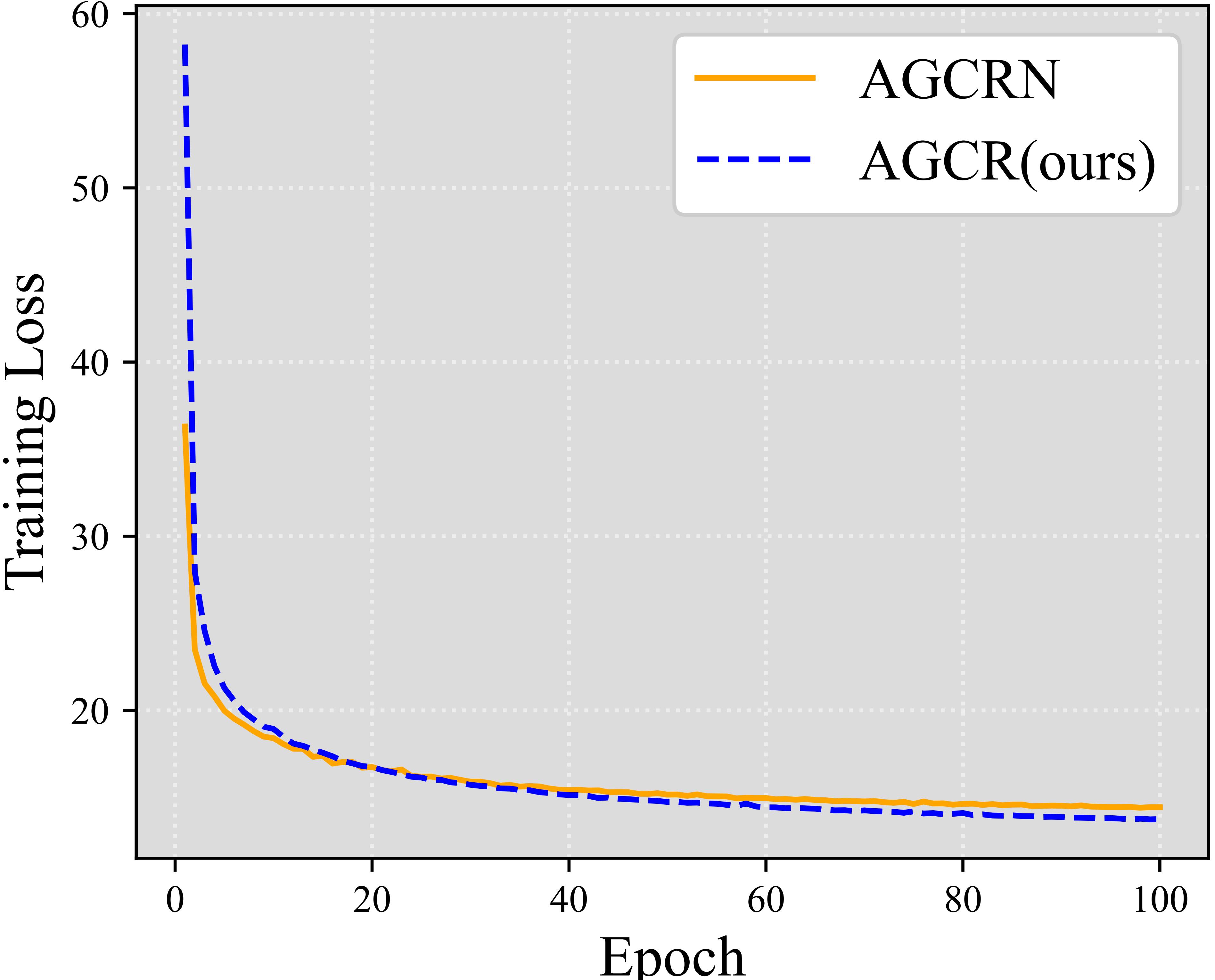}
  }
  \subfloat[MAE Curves]
  {
      \label{fig:mae_curve_sd}\includegraphics[width=0.23\textwidth]{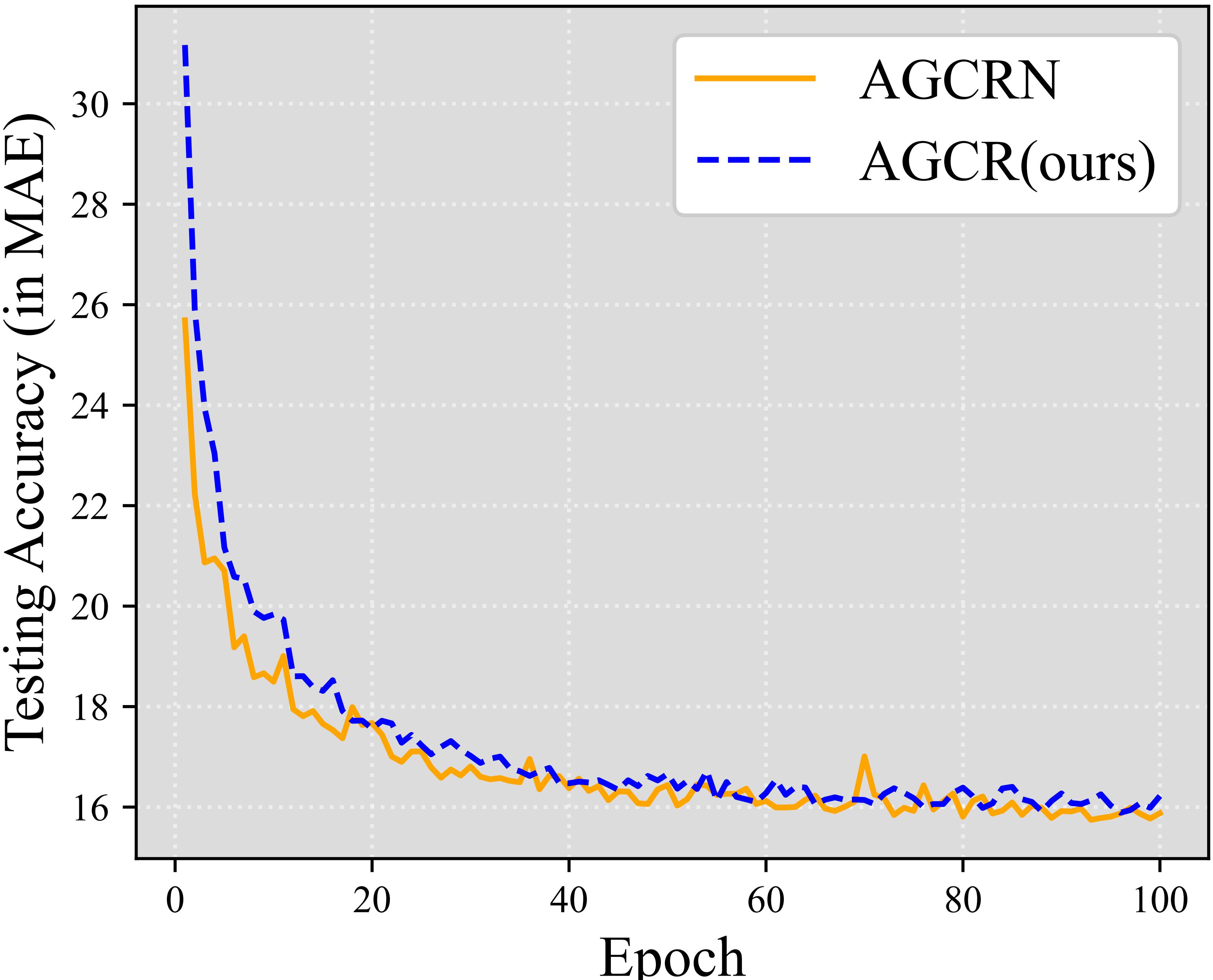}
  }
   \caption{Training loss (a) and testing MAE (b) curve of original AGCRN and AGCRN$^\ast$ trained on SD, respectively.}    
  \label{fig:sd}   
\end{figure}

\subsection{Analysis}\label{sec:anal}
\fakeparagraph{Convergence}
\figref{fig:07} illustrates the training loss and test Mean Absolute Error (MAE) curves of the original AGCRN and AGCRN$^\ast$ under identical hyper-parameter settings on the PEMS07 dataset. Similarly, \figref{fig:sd} presents these curves for the same models on the SD dataset. \Sysname ensures convergence that is as consistent, rapid, and stable as a complete graph model. This feature is particularly advantageous for training on large-scale spatial-temporal data, as it significantly reduces computational overhead without compromising the quality of convergence.

Additionally, the convergence behavior of AGCRN$^\ast$ demonstrates its robustness in capturing complex spatial-temporal dependencies. This attribute is crucial for reliable forecasting in dynamic systems, such as traffic networks, where understanding intricate patterns is key to accuracy.

\fakeparagraph{AGCRN$^\ast$\&GWNET$^\ast$ vs. SOTAs}
AGCRN and GWNET, as representative ASTGNNs introduced between 2019 and 2020, are of significant interest in our study. Our objective is to evaluate the performance of AGCRN and GWNET, particularly when trained using GWT, in comparison with the current state-of-the-art STGNNs. To this end, we selected five advanced STGNNs as baselines: DGCRN \cite{li2023dynamic}, MegaCRN \cite{jiang2023spatio}, STGODE \cite{fang2021spatial}, D$^2$STGNN \cite{shao2022decoupled}, and DSTAGNN \cite{lan2022dstagnn}. These models reflect the most recent trends in the field. DGCRN and MegaCRN, seen as variations of AGCRN, epitomize the latest developments in ASTGNN. STGODE employs neural ordinary differential equations innovatively to effectively model the continuous dynamics of traffic signals. In contrast, DSTAGNN and D$^2$STGNN focus on capturing the dynamic correlations among sensors in traffic networks. From the results presented in \tabref{tab:vs-sotas}, we make the following observations: \textit{i)} ASTGNNs such as DGCRN and MegaCRN consistently exhibit strong performance across most benchmarks. However, their intricate model designs limit scalability, particularly in larger datasets like GLA and CA. \textit{ii)} Methods introduced four years ago, such as AGCRN when trained within GWT-AGCN (\ie AGCRN$^\ast$), continue to demonstrate robust performance across various evaluated datasets. Remarkably, they achieve state-of-the-art performance on specific datasets including GBA, GLA, and CA. These findings suggest that GWT-AGCN could play a crucial role in the development of scalable ASTGNNs for future research.

\fakeparagraph{Impact of averaged initialization of node embedding $e_{c}$}
In this study, we employ AGCRN$^\ast$ as a benchmark to evaluate the impact of averaged initialization of node embedding $e_{c}$. \tabref{tab:impact} presents a comparative analysis between AGCRN$^\ast$ and \textit{AGCRN}$^\ast$, \ie random initialization of $e_{c}$. The results indicate that AGCRN$^\ast$ consistently outperforms \textit{AGCRN}$^\ast$ in terms of forecasting accuracy. This finding underscores the significance of deliberate initialization strategies for $e_{c}$ in enhancing the predictive performance of the model.

\fakeparagraph{Comparison with AGS}
\begin{table*}[htbp]
\caption{Comparative experimental results between AGS and our method within AGCRN.}
\label{tab:ags}
\begin{tabular}{l|llllllllllll}
\Xhline{1pt}
                         & \multicolumn{3}{l}{PEMS07}                                                                                                                                                                     & \multicolumn{3}{l}{SD}                                                                                                                                                                          & \multicolumn{3}{l}{GBA}                                                                                                                                                                         & \multicolumn{3}{l}{GLA}          \\ 
\cline{2-13} 
\multirow{-2}{*}{Methods}& MAE                                                           & RMSE                                                          & MAPE                                                           & MAE                                                           & RMSE                                                          & MAPE                                                            & MAE                                                           & RMSE                                                          & MAPE                                                            & MAE   & RMSE           & MAPE    \\ 
\hline
\hline
AGCRN(AGS)@99.7\%        & 24.37                                                         & 39.33                                                         & 9.57\%                                                         & 21.83         & 37.10         & 15.03\%        & 23.61         & 38.32         & 17.19\%          & 22.23 & 35.55          & 13.39\% \\
AGCRN$^{\star}$          & 20.67         & \textbf{33.95} & 8.67\%         & 18.56         & 31.89          & 13.19\%         & 20.71         & 33.75        & 15.95\%         & 20.15 & \textbf{32.57} & 12.17\% \\
AGCRN$^*$                & \textbf{20.57} & 34.42        & \textbf{8.59\%} &  \textbf{18.13} & \textbf{30.92} & \textbf{12.41\%} & \textbf{20.51} & \textbf{33.57} & \textbf{15.33\%} & 20.04 & 32.72          & \textbf{12.03\%} \\ 
\Xhline{1pt}
\end{tabular}
\end{table*}
We compared our method with AGS, the state-of-the-art approach, to validate its superiority. The performance of AGS with a sparsity of 99.7\% is reported on PEMS07, SD, GBA and GLA, while the sparsity of our method is 99.8\%/99.7\%/ \\99.99\%/99.99\% for PEMS07/SD/GBA/GLA. Since AGS does not provide an implementation on GWNet, we only report the results for AGCRN. The lack of CA results is due to AGS encountering out-of-memory (OOM) issues. From \tabref{tab:ags}, we can see that our method significantly outperforms AGS.

\fakeparagraph{Perturbed $\mathcal{T}^\star$}
We attribute the effectiveness of $\mathcal{T}^\star$ to its robust connectivity, which is crucial for ASTGNN's ability to model global spatial dependencies. To further validate this perspective, we introduce a perturbation process illustrated in \figref{fig:peturb} to $\mathcal{T}^\star$, resulting in $\mathcal{\Tilde{T}}^\star$, according to the following steps:
\begin{enumerate}
    \item For a given $\mathcal{T}^\star$ with $N$ nodes, we randomly remove $M$ edges between the center node and the leaf node. The perturbation ratio $p$ of the removed edges is defined as $\frac{M}{N-1}$.
    \item Subsequently, we randomly add $M$ new edges, connecting previously isolated nodes.
\end{enumerate}
These steps intentionally disrupt the original connectivity in $\mathcal{T}^\star$, while ensuring that the overall sparsity of the network remains constant.
\figref{fig:connect} show the MAE curves of AGCRN trained via $\mathcal{\Tilde{T}}^\star$ of a ratio $p$ from 0 to 50\%. We can see that as 
$p$ increases, the accuracy of the model decreases. This indicates the importance of preserving the graph's connectivity to model global spatial dependencies

\begin{figure}[t]
    \centering
    \includegraphics[width=0.8\linewidth]{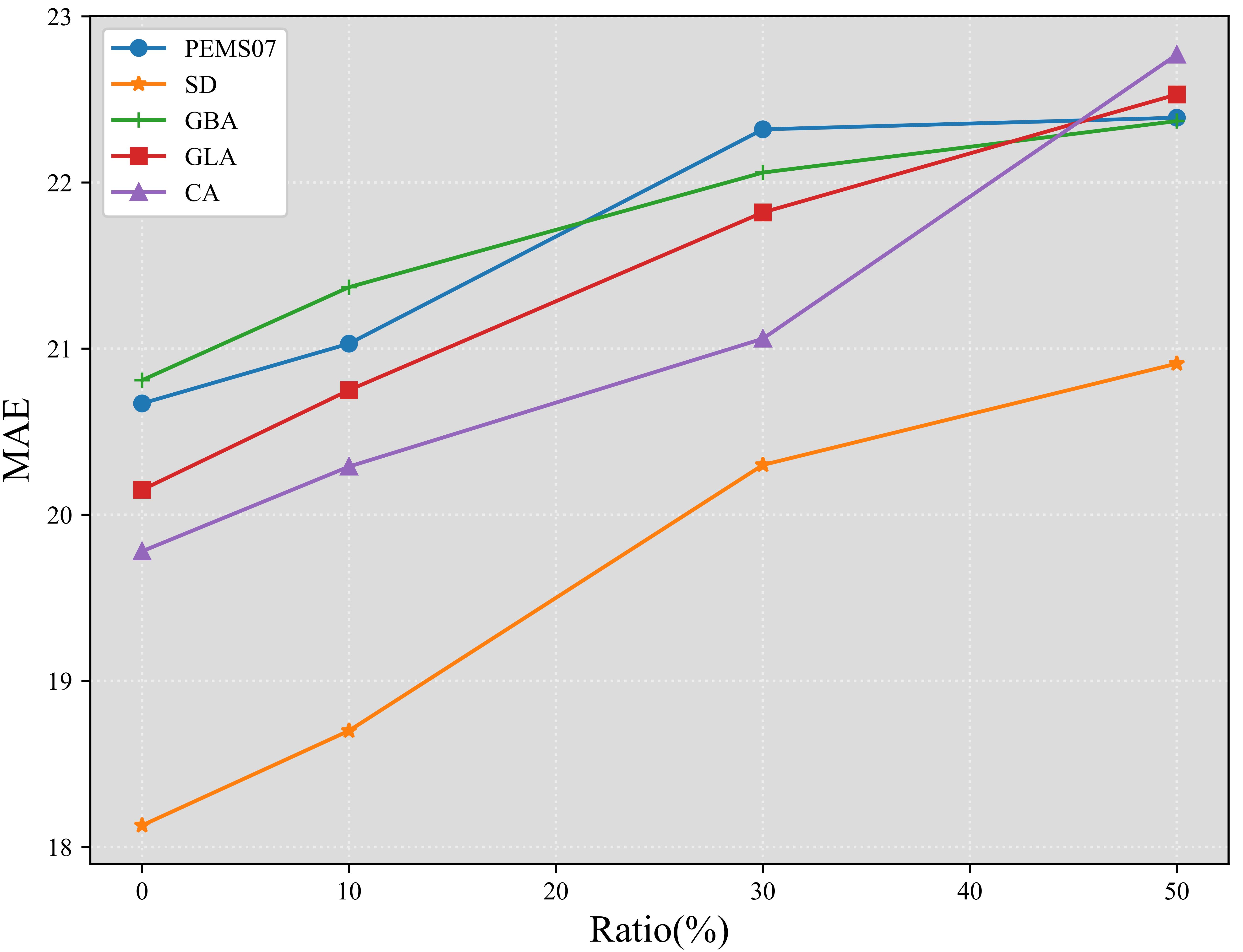}
    \caption{Testing accuracies, measured in MAE, for AGCRN on the PEMS07, SD, GBA, GLA, and CA datasets, with a perturbation ratio of $p$ ranging from 0\% to 50\%.}
    \label{fig:connect}
\end{figure}

\begin{table}[htbp]
\small
\caption{AGCRN$^\ast$ and GWNET$^\ast$ compared with current state-of-the-art STGNNs. We report the results averaged in 12 horizons. Results without standard deviations are sourced from published papers.}
\begin{tabular}{l|llll}
\Xhline{1pt}
Dataset                 & Method  & MAE                 & RMSE                & MAPE(\%)            \\ \hline\hline
\multirow{7}{*}{PEMS07} & AGCRN$^\ast$   & 20.57$\pm$0.11          & 34.42+0.05          & \textbf{8.59$\pm$0.13}  \\
                        & GWNET$^\ast$   & \textbf{19.80$\pm$0.19} & 32.84$\pm$0.25          & 8.62$\pm$0.11           \\
                        & DGCRN   & 20.5$\pm$0.23           & 33.32$\pm$0.23          & 8.45$\pm$0.20           \\
                        & MegaCRN & 19.86$\pm$0.36          & \textbf{32.69$\pm$0.36} & 8.62$\pm$0.21           \\
                        & STGODE  & 22.99               & 37.54               & 10.14               \\
                        & D2STGNN & 20.50               & 33.08               & 8.42                \\
                        & DSTAGNN & 20.50               & 34.51               & 9.01                \\ \hline
\multirow{7}{*}{SD}     & AGCRN$^\ast$   & 18.13$\pm$0.30          & 30.92+0.32          & \textbf{12.41$\pm$0.10} \\
                        & GWNET$^\ast$   & 17.89$\pm$0.18          & 30.06+0.19          & 12.57$\pm$0.21          \\
                        & DGCRN   & 18.02               & 30.09               & 12.07               \\
                        & MegaCRN & \textbf{17.76$\pm$0.21} & \textbf{29.62+0.17} & 12.69$\pm$0.13          \\
                        & STGODE  & 21.79               & 35.37               & 13.22               \\
                        & D2STGNN & 17.85               & 29.51               & 11.54               \\
                        & DSTAGNN & 21.82               & 34.68               & 14.40               \\ \hline
\multirow{7}{*}{GBA}    & AGCRN$^\ast$   & \textbf{20.51$\pm$0.13} & \textbf{33.57+0.27} & \textbf{15.3$\pm$0.17}  \\
                        & GWNET $^\ast$  & 21.01$\pm$0.07          & 33.38$\pm$0.30          & 17.66$\pm$0.15          \\
                        & DGCRN   & 20.91               & 33.83               & 16.88               \\
                        & MegaCRN & 20.69$\pm$0.17          & 33.61$\pm$0.14          & 15.52$\pm$0.09          \\
                        & STGODE  & 21.79               & 35.37               & 18.26               \\
                        & D2STGNN & 20.71               & 33.65               & 15.04               \\
                        & DSTAGNN & 23.82               & 37.29               & 20.16               \\ \hline
GLA                     & AGCRN$^\ast$   & \textbf{20.04$\pm$0.07} & \textbf{32.72$\pm$0.19} & \textbf{12.03$\pm$0.18} \\
                        & GWNET$^\ast$   & 21.08$\pm$0.05          & 33.03$\pm$0.11          & 13.13$\pm$0.20          \\
                        & STGODE  & 21.49               & 36.14               & 13.72\%             \\
                        & DSTAGNN & 24.13               & 38.15               & 15.07\%             \\ \hline
CA                      & AGCRN$^\ast$   & \textbf{19.78$\pm$0.08} & \textbf{31.98+0.29} & \textbf{15.59$\pm$0.43} \\
                        & GWNET$^\ast$   & 21.08$\pm$0.05          & 33.03$\pm$0.11          & 13.13$\pm$0.20          \\
                        & STGODE  & 20.77               & 36.60               & 16.80               \\ \Xhline{1pt}
\end{tabular}
\label{tab:vs-sotas}
\end{table}

\begin{table}[htbp]
\small
\caption{Ablation study of averaged initialization of $e_{c}$. We report the results averaged in 12 horizons. }
\begin{tabular}{l|llll}
\Xhline{1pt}
Dataset                 & Method & MAE                 & RMSE                & MAPE(\%)            \\ \hline \hline
\multirow{2}{*}{PEMS07} & AGCRN$^\ast$  & \textbf{20.57$\pm$0.11}          & \textbf{34.42$\pm$0.05}          & \textbf{8.59$\pm$0.13}  \\
                        & \textit{AGCRN}$^\ast$  & 20.93$\pm$0.17          & 35.17$\pm$0.12          & 8.87$\pm$0.26           \\ \hline
\multirow{2}{*}{SD}     & AGCRN$^\ast$  & \textbf{18.13$\pm$0.30}          & \textbf{30.92+0.32}          & \textbf{12.41$\pm$0.10} \\
                        & \textit{AGCRN}$^\ast$  & 18.63$\pm$0.28          & 32.19$\pm$0.29          & 12.67$\pm$0.24          \\ \hline
\multirow{2}{*}{GBA}    & AGCRN$^\ast$  & \textbf{20.51$\pm$0.13} & \textbf{33.57+0.27} & \textbf{15.3$\pm$0.17}  \\
                        & \textit{AGCRN}$^\ast$  & 21.16$\pm$0.37          & 34.28+0.30          & 16.12$\pm$0.35          \\ \hline
\multirow{2}{*}{GLA}    & AGCRN$^\ast$  & \textbf{20.04$\pm$0.07} & \textbf{32.72+0.19} & \textbf{12.03$\pm$0.18} \\
                        & \textit{AGCRN}$^\ast$  & 20.23$\pm$0.09          & 33.16$\pm$0.20          & 12.32$\pm$0.17          \\ \hline
\multirow{2}{*}{CA}     & AGCRN$^\ast$  & \textbf{19.78$\pm$0.08} & \textbf{31.98+0.29} & \textbf{15.59$\pm$0.43} \\
                        & \textit{AGCRN}$^\ast$  & 20.01+0.06          & 32.64$\pm$0.19          & 16.73$\pm$0.56          \\ \Xhline{1pt}
\end{tabular}\label{tab:impact}
\end{table}

\begin{figure}[t]
    \centering
    \includegraphics[width=0.9\linewidth]{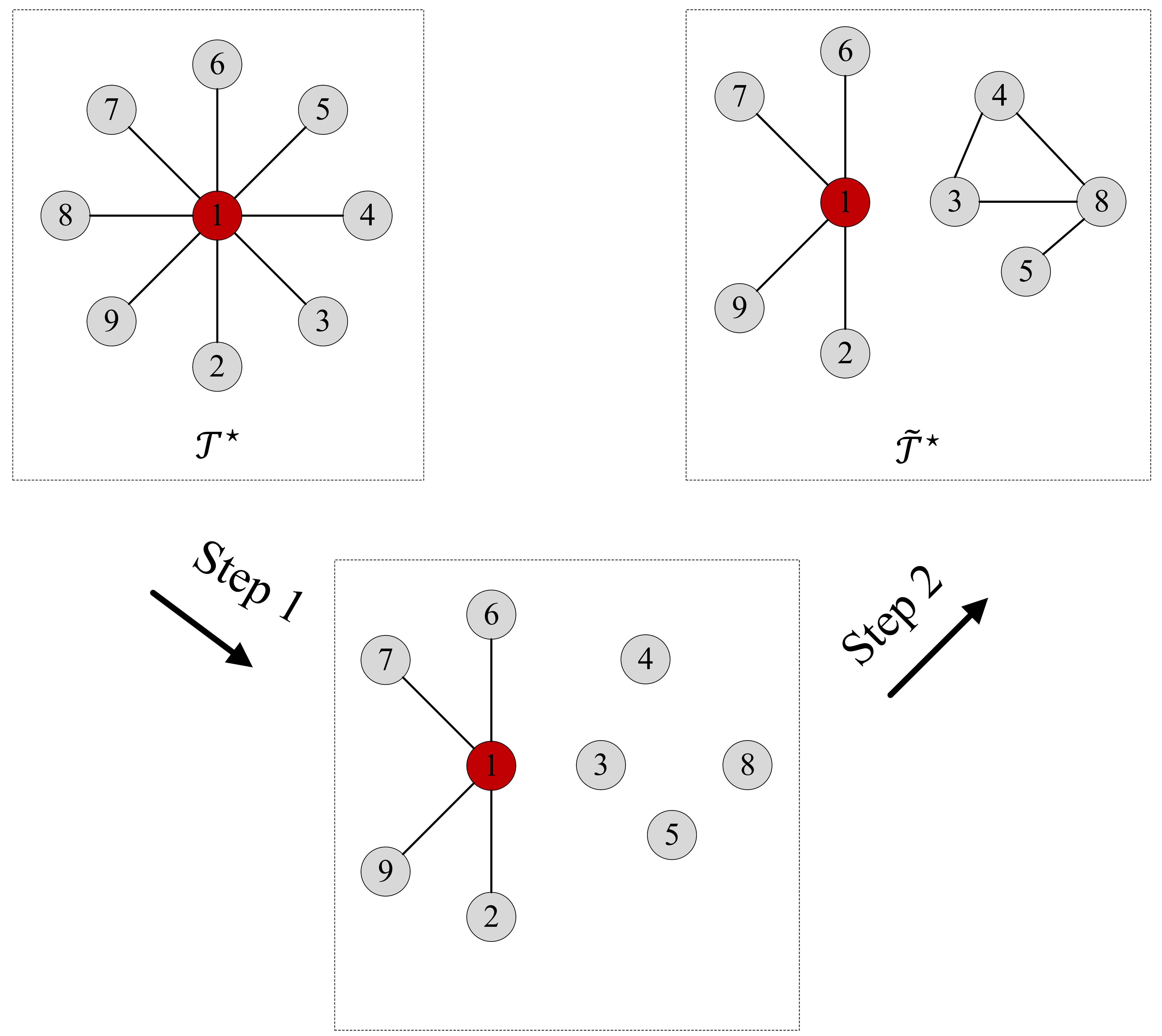}
    \caption{Perturbation process with a perturbation ratio of $p$, with a pre-specified node number.}
    \label{fig:peturb}
\end{figure}
\section{Conclusion}

This paper introduces a novel approach in the realm of ASTGNNs by leveraging the GWT concept, inspired by the Lottery Ticket Hypothesis. This method markedly reduces the computational complexity of ASTGNNs, transitioning from a quadratic to a linear scale, thereby streamlining their deployment. Our innovative strategy of adopting a star topology for GWT, without necessitating exhaustive training cycles, maintains high model performance with significantly lower computational demands. Empirical validations across various datasets underscore our method's capability to achieve performance on par with full models, but at a fraction of the computational cost. This breakthrough not only underscores the existence of efficient sub-networks of the spatial graphs within ASTGNNs, but also extends the applicability of the Lottery Ticket Hypothesis to scenarios where resources are limited. Consequently, this work represents a significant leap forward in the optimization and practical application of graph neural networks, particularly in environments where computational resources are constrained.
In the future, we will develop new STGNNs based on \sysname, aimed at long-term spatial-temporal forecasting.
\bibliographystyle{ACM-Reference-Format}
\clearpage


\bibliography{cites}
\appendix
\setcounter{theorem}{0}
\section{appendix}

\subsection{Proof of Proposition 1}\label{append:proof}
Initially, in the case of $N = 3$ as illustrated in \figref{fig:3-sp-tree}, all spanning trees of this complete graph meet the diameter $r= 2$, and satisfy definition of star spanning tree in Hypothesis 1. Their count corresponds to the number of nodes $N$ in the complete graph.


\begin{figure}[t]
\centering 
\includegraphics[width=0.35\textwidth]{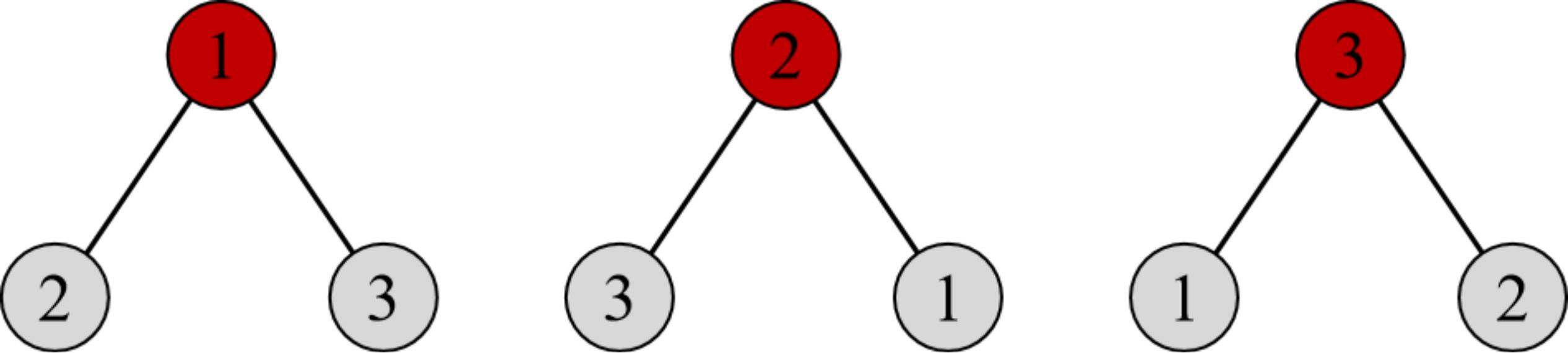}
\caption{
    The 3-order spanning trees of the complete graph are all star spanning trees.}
\label{fig:3-sp-tree}
\end{figure}

Subsequently, assuming $N = k-1$, the complete graph $\mathcal{K}_{k-1}$ aligns with this conclusion, and the star spanning tree is $\mathcal{T}^{\star}_{k-1}$.

    

In the scenario where $N=k$, the original graph is equivalent to inserting a new node into $\mathcal{T}^{\star}_{k-1}$. \figref{fig:add} shows two possible scenarios.

\begin{figure}[t]
\centering 
\includegraphics[width=0.4\textwidth]{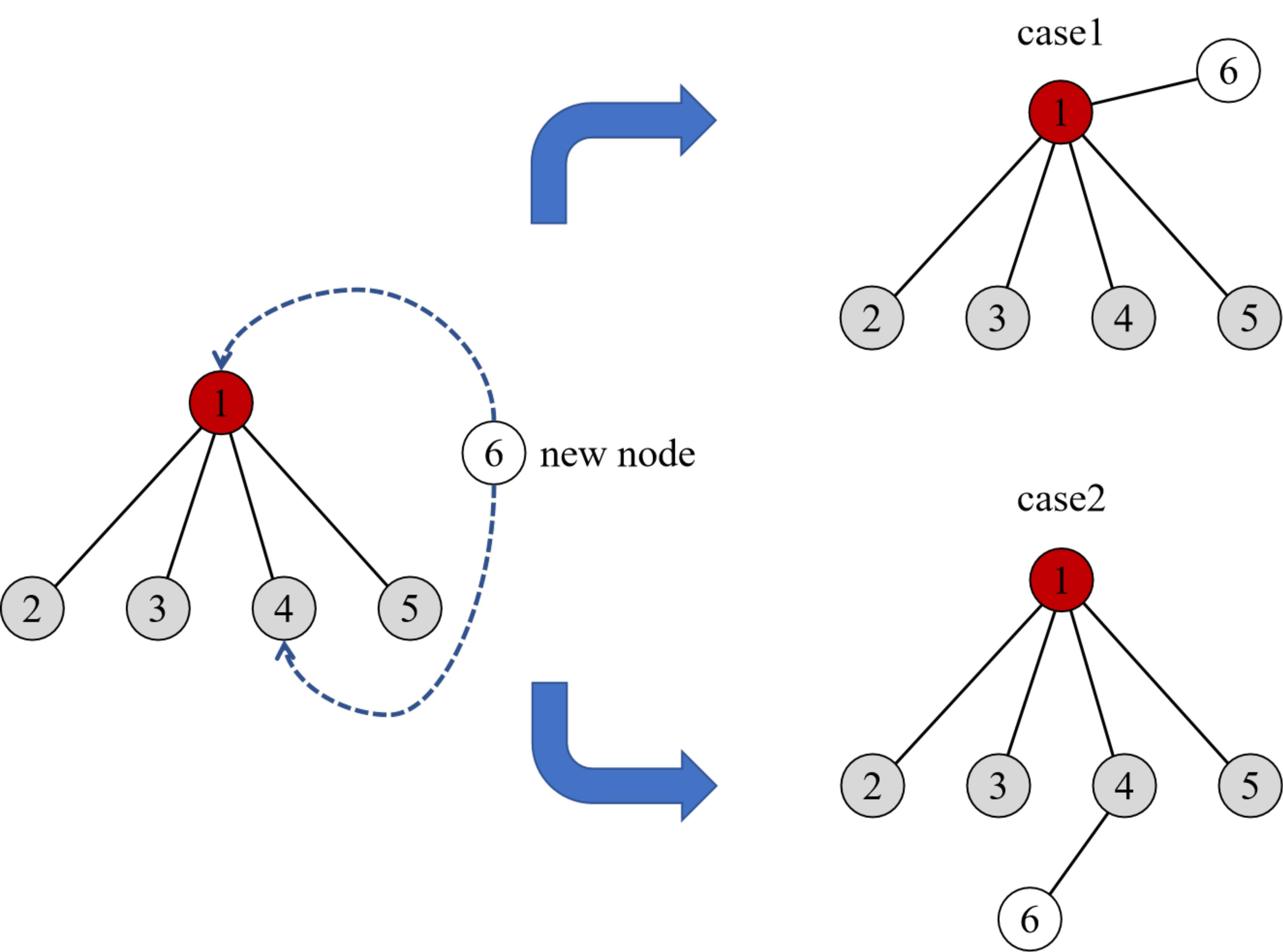}
\caption{Adds the new node to the native tree.}
\label{fig:add}
\end{figure}

Only in the first scenario, does the spanning tree $\mathcal{T}^{\star}_{k}$ meet the diameter $r= 2$. The second scenario will increase some paths that are longer than 2. For the spanning tree $\mathcal{T}^{\star}_{k}$ formed in the first scenario, it still conforms to the definition of
star spanning tree in Hypothesis 1.

\subsection{Justify the effectiveness of star topology theoretically}\label{append:graph}

Given two graphs $\mathcal{G}$ and $\mathcal{\hat{G}}$, if $\boldsymbol{L}_{\mathcal{G}} \preceq \boldsymbol{L}_{\mathcal{\hat{G}}}$, we denote this as: $\mathcal{G}\preceq\mathcal{\hat{G}}$. 

Here, $\boldsymbol{L}_{\mathcal{G}}$ and $\boldsymbol{L}_{\mathcal{\hat{G}}}$ represent the Laplacians of $\mathcal{G}$ and $\mathcal{\hat{G}}$, respectively. The symbol $\preceq$ denotes the Loewner partial order, applicable to certain pairs of symmetric matrices.

The Courant-Fisher Theorem provides that:

\begin{equation}
    \lambda_i(A)=\max_{S:\dim(S) = i}\min_{x\in S} \frac{x^TAx}{x^Tx}.
\end{equation}

Thus, assuming $\lambda_1,\ldots, \lambda_N$ are the eigenvalues of $\boldsymbol{L}_{\mathcal{G}}$ and $\tilde{\lambda}_1,\ldots,\tilde{\lambda}_n$ are the eigenvalues of $\boldsymbol{L}_{\mathcal{\hat{G}}}$. The relation $L_{\mathcal{G}} \preceq L_{\mathcal{\hat{G}}}$ means for all i, $\lambda_{i}\leq \hat{\lambda}_i$.

\fakeparagraph{Graph Spectral Similarity} \cite{Daniel}\cite{Batson2013} If $L_{\mathcal{\hat{G}}}/ \sigma \preceq L_{\mathcal{G}} \preceq \sigma L_{\mathcal{\hat{G}}}$, we say graphs $\mathcal{G}$ and $\mathcal{\hat{G}}$ are $\sigma$-spectral similar. Thus, $\mathcal{\hat{G}}$ is a $\sigma$-approximation of $\mathcal{G}$.

Based on spectral graph theory  \cite{Daniel}\cite{Batson2013}, if a graph is a $\sigma-approximation$ of another one. We mean they have similar eigensystems, therefore with similar properties. Thus, if $\boldsymbol{L}_{\mathcal{T}_N} / \sigma \preceq L_{\mathcal{K}_N} \preceq\sigma L_{\mathcal{T}_N}$, 
$\mathcal{K}_N$ and $\mathcal{T}_N$ have similar properties. Such a $\mathcal{T}_N$ can effectively replace $\mathcal{K}_N$ to learn a good representation, where the edges of $\mathcal{T}_N$ are much fewer than those of $\mathcal{K}_N$.  Below, we will prove that $\mathcal{T}_N$ is a $\sigma-approximation$ of $\mathcal{K}_N$.
\begin{lemma}
    The laplacian of $\mathcal{K}_N$ has eigenvalue 0 with multiplicity 1 and $N$ with multiplicity $N-1$.
\end{lemma}

\begin{proof}[Proof of Lemma 1.]To compute the non-zero eigenvalues, let $\boldsymbol{\psi}$ be any non-zero vector orthogonal to the all-1s vector, so

\begin{equation}
    \sum_a\boldsymbol{\psi}(a)=0.
\end{equation}

The Laplacian Matrix of a weighted graph $\mathcal{G}=(V,E,w), w:E\to\mathbb{R}^{+}$ , is designed to capture the Laplacian quadratic form:

\begin{equation}\label{eq:lg}
\begin{split}
    (\boldsymbol{L}_{\mathcal{G}}\boldsymbol{x})(a)&=\sum_{(a,b)\in E}w_{a,b}(\boldsymbol{x}(a)-\boldsymbol{x}(b))
    \\&=d(a)\boldsymbol{x}(a)-\sum_{(a,b)\in E}w_{a,b}\boldsymbol{x}(b).
\end{split}
\end{equation}

We now compute the first coordinate of $\boldsymbol{L}_{\mathcal{K}_n}\boldsymbol{\psi}$. Using the expression for the action of the Laplacian as an operator, we find

\begin{equation}
\begin{split}
    (\boldsymbol{L}_{\mathcal{K}_n}\boldsymbol{\psi}) (1)&=\sum_{v\geq2}(\boldsymbol{\psi}(1)-\boldsymbol{\psi}(b))
    \\&=(n-1)\boldsymbol{\psi}(1)-\sum_{v=2}^n\boldsymbol{\psi}(b)=n\boldsymbol{\psi}(1).
    \end{split}
\end{equation}

As the choice of coordinate was arbitrary, we have $\boldsymbol{L\psi}=N\boldsymbol{\psi}$. So, every vector orthogonal to the all-1s vector is an eigenvector of eigenvalue $N$.
\end{proof}

\begin{lemma} Let $\mathcal{G}=(\mathcal{V},\mathcal{E})$ be a graph, and let $a$ and $b$ be vertices of degree one that are both connected to another vertex $c$. Then, the vector $\boldsymbol{\psi=\delta_{a}-\delta_{b}}$ is an eigenvector of $\boldsymbol{L}_{G}$ of eigenvalue 1
\end{lemma}

\begin{proof}[Proof of Lemma 2.] Just multiply $\boldsymbol{L}_{G}$ by $\boldsymbol{\psi}$, and check (using \eqref{eq:lg}) vertex-by-vertex that it equals $\boldsymbol{\psi}$.

As eigenvectors of different eigenvalues are orthogonal, this implies that $\boldsymbol{\psi}_a = \boldsymbol{\psi}_b$ for every eigenvector with eigenvalue different from 1.
\end{proof}

\begin{lemma} The laplacian of $\mathcal{T}_N$ has eigenvalue 0 with multiplicity 1, eigenvalue 1 with multiplicity $N-2$, and eigenvalue $N$ with multiplicity 1.
\end{lemma}

\begin{proof}[Proof of Lemma 3.]
Applying Lemma 2.1 to vertices $i$ and $i + 1$ for $2 \leq i < N$ , we find $N - 2$ linearly independent eigenvectors of the form $\boldsymbol{\delta_i-\delta_{i+1}}$, all with eigenvalue 1. As 0 is also an eigenvalue, only one eigenvalue remains to be determined.
Recall that the trace of a matrix equals both the sum of its diagonal entries and the sum of its eigenvalues. We know that the trace of $\boldsymbol{L}_{\mathcal{T}_n}$ is $2N - 2$, and we have identified $N - 1$ eigenvalues that sum to $N - 2$. So, the remaining eigenvalue must be $N$.
\end{proof}

From Lemma 1-3, we deduce:

\begin{lemma}
    $\mathcal{T}_N$ is an N-approximation of  $\mathcal{K}_N$.
\end{lemma} 
\begin{proof}[Proof of Lemma 5.]
    Assume $\lambda_1,\ldots, \lambda_N$ are the eigenvalues of  The Laplacian of $\mathcal{K}_N$, and $u_1,\ldots, u_N$ are the eigenvalues of the Laplacian of $\mathcal{T}_N$. For i =1, $\lambda_i = N$, $u_i = N$, satisfying $u_i/N \leq \lambda\_i \leq N u_i$. For $2\leq i \leq N-1$, $\lambda_i = N$, $u_i =1$, satisfying $u_i/N \leq \lambda_i \leq N u_i$. For i =N, $\lambda_i = 0$, $u_i = 0$, satisfying $u_i/N \leq \lambda_i \leq N u_i$. Thus, for all i, $u_i/N \leq \lambda_i \leq N u_i$, i.e., $L_{\mathcal{T}_N}/ N \preceq L_{\mathcal{K}_N} \preceq N L_{\mathcal{T}_N}$.
\end{proof}

In conclusion, we have theoretically proven that star topology $\mathcal{T}_N$ is a good approximation of $\mathcal{K}_N$, and therefore, can learn spatiotemporal dependencies effectively.




\end{document}